\documentclass{article}
\usepackage[preprint]{conference}
\usepackage{times}

\usepackage{amsmath,amsfonts,bm}

\def\eqref#1{equation~\ref{#1}}

\def\1{\bm{1}}

\DeclareMathAlphabet{\mathsfit}{\encodingdefault}{\sfdefault}{m}{sl}
\SetMathAlphabet{\mathsfit}{bold}{\encodingdefault}{\sfdefault}{bx}{n}

\usepackage{graphicx}
\usepackage{subcaption}
\usepackage{algorithm}
\usepackage{algorithmic}
\usepackage{amsmath}
\usepackage{amssymb}
\usepackage{amsfonts}
\usepackage{amsthm}
\usepackage{booktabs}
\usepackage{multirow}
\usepackage{float}
\usepackage{placeins}
\usepackage{xcolor}
\usepackage[hidelinks]{hyperref}
\usepackage{url}

\newtheorem{proposition}{Proposition}

\title{SymRegFlow: Symmetry-Regularized Flow Matching for Video World Models}

\author{%
\textbf{Xi Ye\textsuperscript{1} \quad Yuzhu Wang\textsuperscript{1} \quad
Xiaoyang Liu\textsuperscript{1,5} \quad Jiayi Wang\textsuperscript{2,3}}\\
\textbf{Yangyang Xu\textsuperscript{1} \quad Ruyu Wang\textsuperscript{2,3} \quad
Wenlin Chen\textsuperscript{2,4} \quad Duo Su\textsuperscript{1} \quad Jun Zhu\textsuperscript{1}}\\[0.5em]
\parbox[t]{\dimexpr\textwidth-2\tabcolsep\relax}{\normalfont\raggedright
\textsuperscript{1}Dept. of Comp. Sci. and Tech., Tsinghua-Bosch Joint ML Center, Tsinghua University\\
\textsuperscript{2}Bosch Center for Artificial Intelligence (BCAI)\\
\textsuperscript{3}Robert Bosch GmbH\\
\textsuperscript{4}Bosch (China) Investment Ltd.\\
\textsuperscript{5}University of Chinese Academy of Sciences
}
}

\hypersetup{
  pdftitle={SymRegFlow: Symmetry-Regularized Flow Matching for Video World Models},
  pdfauthor={Xi Ye, Yuzhu Wang, Xiaoyang Liu, Jiayi Wang, Yangyang Xu, Ruyu Wang, Wenlin Chen, Duo Su, Jun Zhu},
  pdfsubject={Video world models},
  pdfkeywords={flow matching, video world models, symmetry regularization}
}

\begin{document}

\maketitle

\begin{abstract}
Flow-matching-based multi-view world models generate realistic videos, but are commonly
restricted to fixed camera rigs. Extending them to continuously varying camera poses requires paired
pose--video observations with dense pose coverage, which are costly to acquire. We introduce
\emph{SymRegFlow}, a symmetry-regularized flow-matching framework for multi-view-consistent video
generation across continuous viewpoints without ground-truth novel-view RGB supervision.
For each target pose, SymRegFlow geometrically warps source views into noisy anchors and combines
masked dual-anchor supervision with cross-anchor denoising-output consistency to mitigate
anchor-specific errors. Under an affine Gaussian surrogate, we prove that suitable consistency
regularization recovers the clean-reference optimum at fixed noise levels, strictly outperforming
single- and merged-anchor baselines. Experiments on Cosmos-Drive-Dreams and nuScenes demonstrate
high-quality, multi-view-consistent autonomous-driving video generation: on nuScenes, SymRegFlow
achieves the lowest FVD and FVMD among the evaluated baselines, reducing FVD by over 31\% relative
to the best baseline, and source-conditioned inference also attains the best FID and instance
preservation.
\end{abstract}

\section{Introduction}
Video world models have become powerful generative simulators of dynamic environments.
In autonomous driving, recent systems jointly synthesize the RGB streams of a surround-view
camera rig, conditioned on maps, object layouts, and text
\citep{magicdrive,vista,cosmos}. Joint multi-view learning can implicitly capture spatial
consistency across views. We refer to the desired property as video $\mathrm{SE}(3)$ symmetry:
under a rigid camera-pose change $g\in\mathrm{SE}(3)$, the scene content and dynamics remain
unchanged, while the generated RGB video transforms consistently with the induced reprojection on
co-visible regions \citep{dupont2020equivariant,equivlightfield}. Because $\mathrm{SE}(3)$ is
continuous, this consistency should extend across camera poses; current models, however, support
only the finite viewpoints represented by their training rigs.

Continuous-view generation supports sensor simulation, closed-loop autonomy evaluation
\citep{unisim}, training-data augmentation \citep{neurad}, and cinematographic camera control
\citep{cameractrl,recammaster}. However, nuScenes and Waymo Open \citep{nuscenes,waymo} provide
synchronized videos only from fixed calibrated cameras, leaving intermediate viewpoints and
novel camera heights or baselines without paired targets. ReCamMaster and SynCamMaster obtain
synchronized multi-camera training videos from Unreal Engine \citep{recammaster,syncammaster},
incurring high generation costs, a sim-to-real gap, and unmodeled simulator artifacts.

To avoid simulator-generated supervision, we use off-the-shelf 3D perception models to
warp observed source views to sampled novel poses, producing noisy anchor videos affected by
geometry errors, occlusions, and resampling. We introduce \emph{SymRegFlow}, which combines
masked dual-anchor supervision with cross-anchor denoising consistency under shared diffusion
noise, while retaining observed-view supervision to preserve pretrained generative capabilities.
This symmetry inductive bias reduces sensitivity to anchor-specific errors without
requiring ground-truth novel-view videos.

SymRegFlow is a fine-tuning objective that keeps the flow-matching training and inference framework.
Our implementation adds continuous-view conditioning, directed cross-view attention (DCVA), and
source-latent replacement for source-conditioned inference. Our contributions are:
\begingroup
\setlength{\leftmargini}{1.5em}
\begin{itemize}
    \item We propose SymRegFlow for cross-view-consistent video generation across continuous
    viewpoints without novel-view RGB supervision, using masked dual-anchor learning and
    cross-anchor denoising consistency.
    \item We prove that, under an affine Gaussian surrogate, suitable consistency regularization
    counteracts anchor-induced variance inflation and recovers the clean-reference optimum at each
    fixed noise level, strictly improving upon single- and merged-anchor training.
    \item We show that on the nuScenes midpoint-view benchmark, SymRegFlow achieves the lowest FVD
    and FVMD among the evaluated baselines (over 31\% lower FVD than the best baseline), while
    source-conditioned inference also attains the best FID and instance preservation.
\end{itemize}
\endgroup

\section{Method}

\subsection{Preliminaries}
\label{sec:preliminaries}
\paragraph{Symmetry and equivariance.}
A symmetry is a group transformation that leaves an object or distribution unchanged. Let
$\mathcal X$ and $\mathcal Y$ denote the input and output spaces, and let
$\rho_{\mathcal X}(g)$ and $\rho_{\mathcal Y}(g)$ denote how a group element $g$ acts on each
space. A map $\mathcal F\colon\mathcal X\rightarrow\mathcal Y$ is equivariant if it commutes with these
actions: $\mathcal F(\rho_{\mathcal X}(g)x)=\rho_{\mathcal Y}(g)\mathcal F(x)$. Invariance is the special case in which
$\rho_{\mathcal Y}(g)$ is the identity \citep{bronstein2021gdl}. Thus, symmetry specifies the
structure to preserve, while equivariance specifies how a model preserves it. For the video
$\mathrm{SE}(3)$ symmetry defined in the introduction, $\rho_{\mathcal X}(g)$ changes the requested
camera pose and $\rho_{\mathcal Y}(g)$ is the induced RGB reprojection. The scene content and
dynamics remain invariant, but its observation changes; the generator should therefore be
camera-pose equivariant on co-visible regions
\citep{dupont2020equivariant,equivlightfield}.

\paragraph{Multi-view flow matching.}
Let $\mathbf{x}=(x^1,\ldots,x^V)$ be a clean, synchronized $V$-view video, where
$x^v\in\mathbb R^{F\times H\times W\times 3}$. Here, $F$, $H$, and $W$ denote the number of frames, frame height, and frame width,
respectively. Let
$\mathcal C=\{(\mathbf T^v,K^v)\}_{v=1}^V$ collect camera poses
$\mathbf T^v\in\mathrm{SE}(3)$ and intrinsics $K^v$. The multimodal condition is
$\mathbf c=(\mathcal C,c_{\mathrm{text}},c_{\mathrm{map}},c_{\mathrm{box}},c_{\mathrm{act}},\ldots)$,
including any available text, maps, boxes, or actions. Given
$\mathbf x\sim p_{\mathrm{data}}(\cdot\mid\mathbf c)$, Gaussian noise
$\boldsymbol\epsilon\sim\mathcal N(0,I)$, and a noise level $t\sim p_{\mathrm{train}}$ drawn from
the training scheduler's distribution on $(0,1)$, the forward path corrupts the video as
$\mathbf Z_t=(1-t)\mathbf x+t\boldsymbol\epsilon$, with
$\mathbf Z_t=(Z_t^1,\ldots,Z_t^V)$, $\mathbf Z_0=\mathbf x$, and
$\mathbf Z_1=\boldsymbol\epsilon$. We denote the denoising output by
$\widehat{\mathbf Z}_{0,\theta}$ and express velocity-matching supervision in clean-prediction space
\citep{li2025backtobasics}:
\begin{equation}
\mathcal L_x=\mathbb E\left[\frac{1}{t^2}
\big\|\widehat{\mathbf Z}_{0,\theta}(\mathbf Z_t,\mathbf c,t)-\mathbf x\big\|_2^2\right].
\label{eq:x-prediction}
\end{equation}
Its component $\widehat Z_{0,\theta}^{v}$ is the denoising output for view $v$.
The corresponding flow velocity is
$\mathbf u_\theta=(\mathbf Z_t-\widehat{\mathbf Z}_{0,\theta})/t$ for $t>0$
\citep{fmatching}. Starting from $\mathbf Z_1\sim\mathcal N(0,I)$, we integrate
$d\mathbf Z_t/dt=\mathbf u_\theta$ backward from $t=1$ to $0$.

Large video world models instead encode each view with a frozen video autoencoder
$\mathcal E,\mathcal D$, yielding clean latents $X^v=\mathcal E(x^v)$ and the multi-view tuple
$\mathbf X=(X^1,\ldots,X^V)$. The latent forward path is
$\mathbf Z_t=(1-t)\mathbf X+t\boldsymbol\epsilon$, with $\mathbf Z_0=\mathbf X$. Sampled clean latents
$\widehat X^v=\widehat Z_0^v$ are decoded as
$\widehat x^v=\mathcal D(\widehat X^v)$ \citep{rombach2022,cosmos}.

\begin{figure*}[t]
\centering
\includegraphics[width=\textwidth]{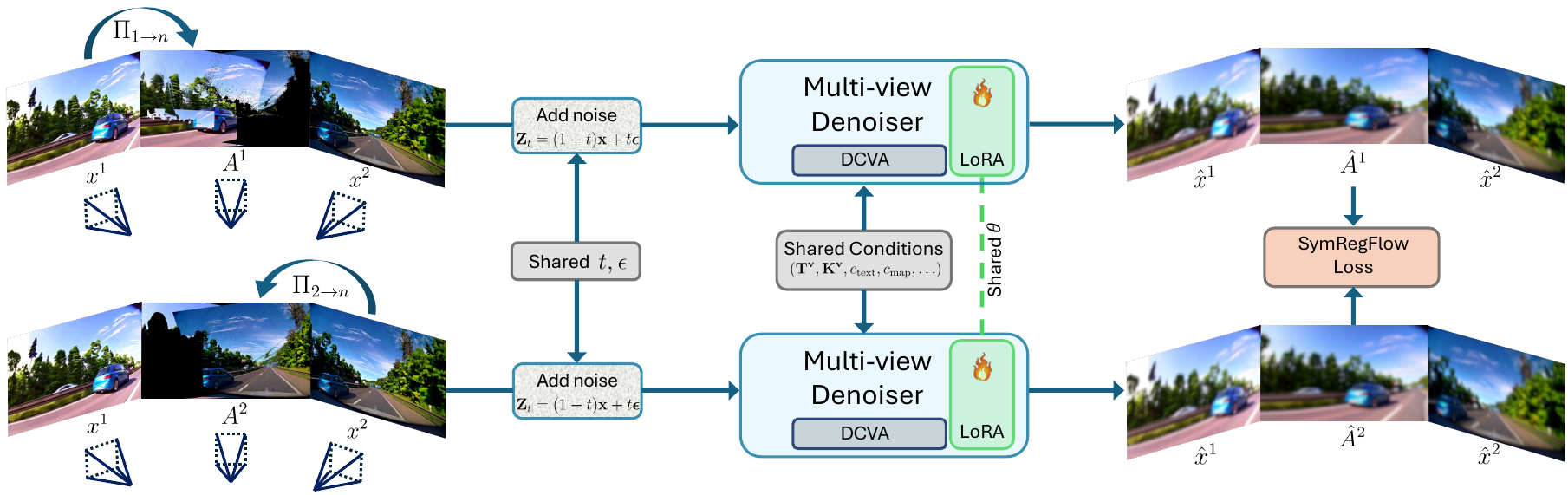}
\caption{SymRegFlow pipeline. Two observed videos $x^1,x^2$ are independently warped to the same requested novel view, yielding noisy anchors $A^1,A^2$. The resulting multi-view pseudo-examples are corrupted with shared noise and fed into the same denoiser with directed cross-view attention (DCVA). VAE of latent flow-matching is omitted from the drawing for clarity.}
\label{fig:symregflow_model}
\end{figure*}

\subsection{SymRegFlow from Noisy Multi-View Anchors}
\label{sec:equiv}
\paragraph{Problem setting.}
Let $x^1$ and $x^2$ denote synchronized ground-truth RGB videos collected at two available
discrete camera views, with poses $\mathbf T^s\in\mathrm{SE}(3)$ and intrinsics $K^s$ for
$s\in\mathcal S=\{1,2\}$. We assume a multi-view flow matching model pretrained on these views.
Our goal is to learn the conditional generator $p(x^1,x^2,x^{\mathrm n}\mid\mathbf T^{\mathrm n},K^{\mathrm n})$ for a
continuous novel view $\mathrm n$, specified by camera pose $\mathbf T^{\mathrm n}\in\mathrm{SE}(3)$ and intrinsics
$K^{\mathrm n}$, by fine-tuning the pretrained model without ground-truth novel-view RGB $x^{\mathrm n}$; the formulation
generalizes naturally to more observed views. For this requested novel view, we encode
$(\mathbf T^{\mathrm n},K^{\mathrm n})$ with a learned embedder $e_\phi$, typically using the pretrained model's
existing pose-conditioning mechanism. Here, $\phi$ denotes the
embedder parameters, learned end-to-end under the SymRegFlow objective defined below.

\paragraph{Geometry-derived noisy anchors.}
Because ground-truth RGB $x^{\mathrm n}$ is unavailable, we form a pseudo-target from each source
using an efficient off-the-shelf 3D perception system such as ViPE \citep{vipe}. We use its
estimated dense depth and camera geometry to lift the source RGB into a colored 3D point cloud,
then project and render it at $(\mathbf T^{\mathrm n},K^{\mathrm n})$. Denoting this pipeline by
$\Pi_{s\rightarrow \mathrm n}$, we obtain
\begin{equation}
(A^s,m^s)=\Pi_{s\rightarrow \mathrm n}(x^s),\qquad s\in\mathcal S.
\label{eq:anchor-corruption}
\end{equation}
Here, $A^s$ is a partial novel-view anchor and $m^s$ its visibility/confidence mask. Writing
$M^s=\operatorname{diag}(m^s)$, its visible pixels satisfy
$M^sA^s=M^s(x^{\mathrm n}+E^s)$, where $E^s$ is the aggregate warping error. Geometry errors,
scene dynamics, and occlusions make the anchors noisy and incomplete. In general, the magnitude
of $E^s$ grows with source--target viewpoint distance, as illustrated by the rotation-dependent
$\mathrm{SE}(2)$ example in Fig.~\ref{fig:toy-study}\subref{fig:toy-anchors}.

\paragraph{Symmetry-regularized fine-tuning.}
Dual-anchor fine-tuning constructs one pseudo-example
$\widetilde{\mathbf x}^{(s)}=(x^1,A^s,x^2)$ for each source $s\in\mathcal S$.
The anchors describe the same target view with different valid supports. Camera-pose equivariance
motivates consistency on shared visible content; we extend this motivation to cross-anchor
denoising agreement over regions supported by either anchor.

To implement this regularizer, we apply the forward corruption from Sec.~\ref{sec:preliminaries} to each
pseudo-example, yielding $\mathbf Z_t^{(s)}$, and predict its clean multi-view output under
$\mathbf c^{\mathrm n}=(\mathbf c,e_\phi(\mathbf T^{\mathrm n},K^{\mathrm n}))$:
\begin{equation}
\widehat{\mathbf Z}_{0,\theta}^{(s)}
=f_\theta\!\left(\mathbf Z_t^{(s)},\mathbf c^{\mathrm n},t\right),
\qquad s\in\mathcal S.
\label{eq:paired-forward}
\end{equation}
Writing $\widehat Z^{v,(s)}:=\widehat Z_{0,\theta}^{v,(s)}$ and
$M^{\cup}=\operatorname{diag}(\max(m^1,m^2))$, we define the cross-anchor consistency regularization loss over regions
supported by either anchor:
\begin{equation}
\mathcal L_{\mathrm{cons}}
 =\big\|\widehat Z^{\mathrm n,(1)}-\widehat Z^{\mathrm n,(2)}\big\|_{M^{\cup}}^2.
\label{eq:symregflow-consistency}
\end{equation}
The union mask includes locations supported by only one anchor, coupling the two predictions
through the available anchor supervision. Locations unsupported by both anchors receive no
consistency penalty.
At each optimization step, the two forward passes share the sampled slot-wise Gaussian noise, noise
level $t$, condition $\mathbf c^{\mathrm n}$, and observed-view inputs (Fig.~\ref{fig:symregflow_model}).

To preserve the pretrained model's capabilities, we retain supervision on all observed views;
to learn the novel-pose embedder $e_\phi$, we supervise each novel-view output with its
geometry-derived anchor:
\begin{align}
\mathcal L_{\mathrm{src}}
&=\frac{1}{2}\sum_{s\in\mathcal S}\sum_{r\in\mathcal S}
\big\|\widehat Z^{r,(s)}-x^r\big\|_2^2,\nonumber\\
\mathcal L_{\mathrm{anc}}
&=\sum_{s\in\mathcal S}\big\|\widehat Z^{\mathrm n,(s)}-A^s\big\|_{M^s}^{2}.
\label{eq:symregflow-components}
\end{align}
The complete objective is
\begin{equation}
\mathcal L=\mathbb E\left[
\frac{1}{t^2}\left(\mathcal L_{\mathrm{src}}+\mathcal L_{\mathrm{anc}}\right)
+\lambda_c\mathcal L_{\mathrm{cons}}\right].
\label{eq:symregflow-objective}
\end{equation}
We backpropagate the combined loss once through both branches, with $\lambda_c$ weighting
consistency. The supervision and consistency objectives extend to additional observed views. At inference, the novel camera is specified through
$e_\phi(\mathbf T^{\mathrm n},K^{\mathrm n})$, without anchors or geometric warping.

\paragraph{Directed Cross-View Attention.}
To protect observed-view fidelity and multi-view consistency from noisy anchors, Directed
Cross-View Attention (DCVA) restricts observed-view queries to observed-view tokens while allowing
novel-view queries to attend to all tokens. Parameter-free masks enforce this observed-to-novel
information flow in every full-attention layer. Sec.~\ref{sec:ablations} evaluates its trade-off
between view fidelity, cross-view agreement, and computational cost.

Appendix~\ref{app:algorithms} gives the two-pass, three-view training procedure in
Algorithm~\ref{alg:symregflow-training}, including the video backbone's prediction parameterization.
Algorithms~\ref{alg:symregflow-nvs} and~\ref{alg:symregflow-generate} describe
\emph{source-cond.} inference (novel-view synthesis conditioned on observed source RGB) and
\emph{joint-generated} inference (joint three-view generation from noise), respectively.

\subsection{$\mathrm{SE}(2)$ Toy Example}
\label{sec:toy-exp}
We illustrate SymRegFlow in an $\mathrm{SE}(2)$ MNIST setting with observed views $x^1$ and $x^2$
at $-90^\circ$ and $0^\circ$, respectively. A pose $\alpha\in(0,1)$ specifies the novel-view angle
$\vartheta_\alpha=-90^\circ(1-\alpha)$; the exact target $x^{\mathrm n}$ is withheld during
fine-tuning and used only for validation. The distortion of anchors $A^1$ and $A^2$ increases with
their respective source--target angular distances, simulating viewpoint-dependent warping errors
(Fig.~\ref{fig:toy-study}\subref{fig:toy-anchors}). We pretrain a pixel-space multi-view
flow-matching model on the observed views and fine-tune it with these anchors under
Eq.~(\ref{eq:symregflow-objective}). We compare against single-anchor baselines ($A^1$-only and
$A^2$-only) and Merged anchor, which uses their mask-weighted fusion $A^{\mathrm M}$.
Appendix~\ref{app:toy} details anchor construction and training.

\begin{figure}[t]
\centering
\captionsetup[subfigure]{skip=2pt}
\begin{subfigure}[b]{0.49\textwidth}
\centering
\includegraphics[width=\linewidth]{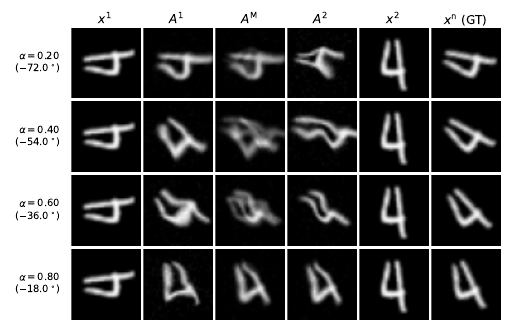}
\caption{}
\label{fig:toy-anchors}
\end{subfigure}\hfill
\begin{subfigure}[b]{0.49\textwidth}
\centering
\includegraphics[width=\linewidth,trim=0 8.6bp 0 0,clip]{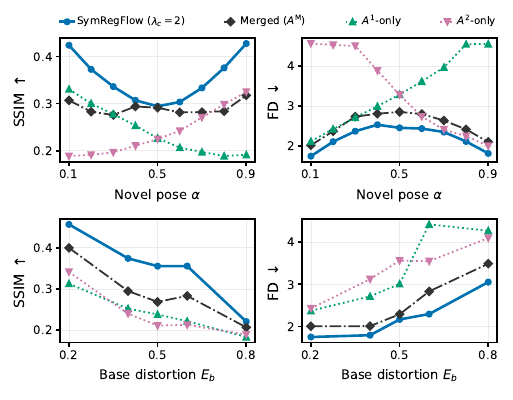}
\caption{}
\label{fig:toy-summary}
\end{subfigure}
\caption{$\mathrm{SE}(2)$ MNIST toy study.
\textbf{(a)} Viewpoint-dependent noisy anchors $A^1,A^2$, their mask-weighted fusion $A^{\mathrm M}$,
and exact target $x^{\mathrm n}$.
\textbf{(b)} SSIM and FD versus novel pose $\alpha$ (top) and base distortion $E_b$ (bottom),
averaged over five distortion levels and nine poses, respectively. FD denotes pixel-space
Fr\'echet distance on flattened $16\times16$ pixel features.}
\label{fig:toy-study}
\end{figure}

\paragraph{Toy-example analysis.}
We evaluate nine withheld poses at five base distortion levels
$E_b\in\{0.2,0.4,0.5,0.6,0.8\}$. All methods start from the same checkpoint, are fine-tuned without
$x^{\mathrm n}$, and are evaluated against exact rotations of held-out digits; Appendix~\ref{app:toy}
provides training and evaluation details and the consistency-weight sweep. SymRegFlow uses DCVA with $\lambda_c=2$, selected by its mean
exact-target pixel-space Fr\'echet distance (FD) over the pose--distortion grid. As shown in
Fig.~\ref{fig:toy-study}\subref{fig:toy-summary}, SymRegFlow achieves both the lowest FD and the
highest SSIM at every evaluated target pose and base distortion level.

Single-anchor performance improves with smaller rotations and cleaner anchors: $A^1$-only is
strongest near $\alpha=0$, and $A^2$-only near $\alpha=1$. Overall, single-anchor fine-tuning performs
worst due to source-specific warping bias, whereas using both anchors improves performance.
SymRegFlow outperforms Merged anchor at every distortion level, with mean FD advantages of
$0.254$ at $E_b=0.2$ and $0.434$ at $E_b=0.8$. Merged anchor remains competitive for 2D rotations,
which omit the geometric, motion, and temporal misalignment of real-world video warps, but
substantially underperforms SymRegFlow on videos (Sec.~\ref{sec:ablations}).

\subsection{Why SymRegFlow Can Improve Noisy-Anchor Fine-Tuning}
\label{sec:theory}
We analyze SymRegFlow's combination of dual-anchor supervision and consistency
regularization in a scalar Gaussian surrogate at fixed condition $\mathbf c^{\mathrm n}$ and
$t\in(0,1)$, with unit masks and fixed source-view predictions. Let
$x^{\mathrm n}\sim\mathcal N(\mu,q)$, $q>0$, and $A^s=x^{\mathrm n}+E^s$.
Assume $x^{\mathrm n},E^1,E^2,\epsilon$ are mutually independent, with centered
Gaussian errors $E^1,E^2$ of positive variance and $\epsilon\sim\mathcal N(0,1)$
shared across branches.
Restrict the novel-view clean predictor to
$f_\theta(z,t)=\mu+\theta[z-(1-t)\mu]$, with scalar slope $\theta$ and conditioning
suppressed. This family contains the optimal clean Gaussian denoiser. On the clean
reference path $Z_t=(1-t)x^{\mathrm n}+t\epsilon$, define the excess flow-matching risk
by subtracting the minimum clean-reference risk:
\begin{equation}
\mathcal R_t^{\mathrm{exc}}(\theta)
:=\frac{1}{t^2}\mathbb E_{x^{\mathrm n},\epsilon}
\left[\bigl(f_\theta(Z_t,t)-x^{\mathrm n}\bigr)^2\right]
-\frac{q}{(1-t)^2q+t^2}.
\label{eq:prop1-clean-excess}
\end{equation}
\begin{proposition}[Clean-reference improvement with SymRegFlow]
\label{prop:clean-reference-improvement}
Under this surrogate, for nonnegative $\lambda_c$, let $\theta_{\lambda_c}$ minimize the
expected SymRegFlow objective in Eq.~(\ref{eq:symregflow-objective}) at the fixed time $t$, and let
$\theta_1,\theta_2,\theta_{\mathrm M}$ be the corresponding unregularized optima for
$A^1$-only, $A^2$-only, and any fixed convex merge $A^{\mathrm M}$. Then
\begin{equation}
\lambda_c^\star(t)=\frac{1}{q(1-t)^2},\qquad
\mathcal R_t^{\mathrm{exc}}\!\left(\theta_{\lambda_c^\star(t)}\right)=0
<\min_{j\in\{1,2,\mathrm M\}}\mathcal R_t^{\mathrm{exc}}(\theta_j).
\label{eq:prop1-strict-improvement}
\end{equation}
The strict improvement holds throughout a neighborhood of $\lambda_c^\star(t)$.
The excess risk decreases up to $\lambda_c^\star(t)$ and increases thereafter.
\end{proposition}
Thus, dual-anchor supervision together with suitable consistency regularization
recovers the clean-reference optimum in this surrogate. Appendix~\ref{app:prop1-proof} gives the proof and numerical
illustration; the guarantee is pointwise in $t$.

\subsection{Extension to Latent Flow Matching}
The preceding formulation uses RGB anchors, whereas latent flow-matching models operate in
autoencoder latent space. Let $\mathcal E$ and $\mathcal D$ be the frozen encoder and decoder, with
$X^s=\mathcal E(x^s)$. We first introduce the induced latent-warp operator
$\mathcal W_{s\rightarrow \mathrm n}:=\mathcal E\circ\Pi_{s\rightarrow \mathrm n}\circ\mathcal D$.
In practice, assuming near-perfect autoencoder reconstruction, we directly warp the source RGB
before encoding:
\begin{equation}
A^s:=\mathcal E\!\left(\Pi_{s\rightarrow \mathrm n}(x^s)\right)
\approx\mathcal W_{s\rightarrow \mathrm n}(X^s).
\label{eq:latent-anchor-construction}
\end{equation}
Here, $\Pi_{s\rightarrow \mathrm n}$ is the pixel-space warping operation defined above,
and $A^s$ is the practical latent anchor used for fine-tuning. Downsampling the corresponding visibility mask to latent
resolution yields $M^s$, and the latent anchor satisfies
\begin{equation}
M^s A^s=M^s(X^{\mathrm n}+E^s),
\label{eq:latent-anchor}
\end{equation}
where $E^s$ denotes the latent-space warping error. During fine-tuning, these latent anchors replace
the RGB anchors in the paired pseudo-examples, using the three loss components in
Eq.~(\ref{eq:symregflow-objective}). The video implementation retains the backbone's supervised
velocity losses, equivalent to $t^{-2}$-weighted clean-output losses, while evaluating consistency
on clean predictions (Algorithm~\ref{alg:symregflow-training}).

\section{Experiments}

Our video models are adapted from the Cosmos-Transfer2.5 multi-view flow-matching backbone
\citep{cosmos25}, which we fine-tune on Cosmos-Drive-Dreams \citep{Ren2025a} and nuScenes \citep{nuscenes} using 29-frame
clips at $1280\times720$ resolution per view, following prior work \citep{dist4d}. We report quantitative evaluation on nuScenes. Following the
depth-based warping of GEN3C \citep{Ren2025b}, we unproject source
frames into 3D using estimated depth and render them at target camera poses to obtain
partial anchors and visibility masks.

\subsection{Qualitative Generation on Cosmos-Drive-Dreams}
\label{sec:cosmos-drive-dreams}
Figure~\ref{fig:cosmos-examples} shows SymRegFlow generating videos at continuously specified novel
camera poses while maintaining visual consistency with the existing camera views. Each top-row novel
view is generated in a separate inference run with the same random seed and scene conditioning.
In the bottom two rows, under the same scene conditioning, seeds 0 and 42 produce different
background trees and cloud patterns, with coherent appearance across views and frames. Unlike
camera control by rendering a fixed 3D reconstruction, SymRegFlow supports both viewpoint variation
and stochastic appearance diversity, which benefits multi-camera data generation for autonomous
driving.

\begin{figure}[htbp]
\centering
\includegraphics[width=\linewidth]{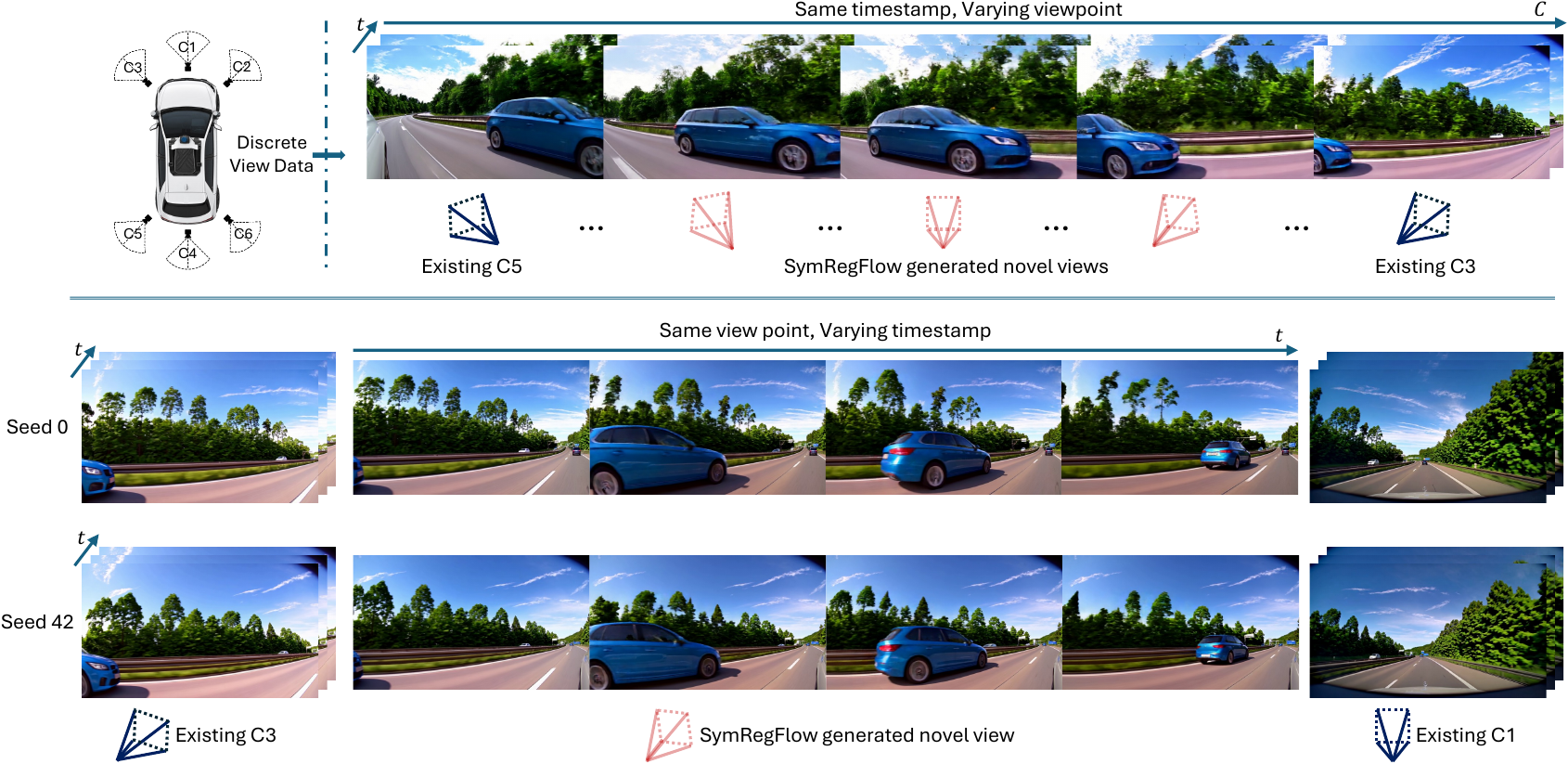}
\caption{Qualitative results on Cosmos-Drive-Dreams. Top: existing and generated novel views at one
timestamp. Bottom: novel-view frames and existing-view sequences for seeds 0 and 42.}
\label{fig:cosmos-examples}
\end{figure}

\FloatBarrier

\subsection{Quantitative Comparison on nuScenes}
\label{sec:nuscenes-comparison}
\paragraph{Evaluation protocol.}
We train and evaluate on the same 30 validation scenes as DiST-4D \citep{dist4d}.
We report novel-view FID \citep{Heusel2017} and FVD \citep{Unterthiner2018} for visual quality,
Fr\'{e}chet Video Motion Distance (FVMD) \citep{fvmd},
VBench subject consistency, background consistency, and temporal flickering \citep{Huang2024}
for temporal quality, VBench-2.0 instance preservation (IP) \citep{Zheng2025}
for entity stability, and MEt3R \citep{met3r} for cross-view consistency.
The detailed protocol is provided in Appendix~\ref{app:nuscenes-protocol}.

\begin{table}[t]
\centering
\caption{Novel-view generation on all 30 nuScenes validation scenes at the interpolated midpoint
($\alpha=0.5$) between two observed camera views.
VBench scores are multiplied by 100.
$^\dagger$MEt3R is computed using generated existing-view videos; other MEt3R scores use captured views.
\textbf{Bold}: best; \underline{underline}: second-best within each column, ranked at full precision.}
\label{tab:quantitative_comparison}
\fontsize{8}{9.5}\selectfont
\setlength{\tabcolsep}{2pt}
\renewcommand{\arraystretch}{1.12}
\begin{tabular*}{\columnwidth}{@{\extracolsep{\fill}}lcccccccc@{}}
\toprule
\textbf{Method}
& FID$\downarrow$
& FVD$\downarrow$
& FVMD$\downarrow$
& \shortstack{VBench\\Subj.$\uparrow$}
& \shortstack{VBench\\Bg.$\uparrow$}
& \shortstack{VBench\\Flicker.$\uparrow$}
& \shortstack{VBench\\IP$\uparrow$}
& MEt3R$\downarrow$ \\
\midrule
DiST-4D \citep{dist4d} & 34.59 & 231.95 & 5381.67 & 88.72 & 91.81 & 95.67 & 92.78 & 0.3362 \\
OmniRe \citep{omnire} & 44.58 & 291.99 & 3104.05 & \textbf{90.20} & \textbf{93.50} & 95.65 & \underline{96.11} & 0.3177 \\
GEN3C \citep{Ren2025b} & 40.28 & 289.24 & 6144.97 & 88.39 & 92.66 & 95.34 & 37.78 & 0.3054 \\
FreeVS \citep{freevs} & \underline{20.60} & 258.53 & 4216.99 & 88.79 & 91.64 & \underline{96.04} & 88.33 & \underline{0.2829} \\
\midrule
\textbf{Ours (\emph{joint-generated})} & 22.00 & \textbf{158.62} & \textbf{2519.17} & 90.13 & 93.32 & \textbf{96.25} & \underline{96.11} & \textbf{0.2818}$^\dagger$ \\
\textbf{Ours (\emph{source-cond.})} & \textbf{19.22} & \underline{159.02} & \underline{2932.87} & \underline{90.17} & \underline{93.32} & 95.82 & \textbf{99.44} & 0.2920 \\
\bottomrule
\end{tabular*}
\end{table}

We evaluate at the midpoint ($\alpha=0.5$), a challenging interpolation pose that maximizes
the distance to the nearest observed view (Sec.~\ref{sec:toy-exp}).
Table~\ref{tab:quantitative_comparison} shows that both inference modes outperform all
baselines in FVD and FVMD, indicating improved video quality and motion consistency.
\emph{Joint-generated} inference achieves the lowest FVD ($158.62$) and FVMD ($2519.17$) and the
highest VBench temporal-flickering score ($96.25$), with subject and background consistency
close to OmniRe. \emph{Source-cond.} inference improves FID from $22.00$ to $19.22$, surpassing
FreeVS ($20.60$), while retaining similar FVD ($159.02$) but yielding higher FVMD ($2932.87$)
and a lower temporal-flickering score ($95.82$).
\emph{Source-cond.} inference also achieves the highest instance-preservation score ($99.44$), which is
critical for autonomous driving.
Because \emph{joint-generated} content can differ from dataset observations under the same scene
conditioning (Fig.~\ref{fig:cosmos-examples}), its MEt3R ($0.2818$) is computed against the
generated existing views and measures internal cross-view agreement, whereas \emph{source-cond.} MEt3R
($0.2920$) measures agreement with captured views.
The latter trails FreeVS ($0.2829$), plausibly because our source conditioning relies on
training-free latent replacement, whereas FreeVS injects reprojected source appearance through
pseudo-images during training and inference \citep{freevs}.

Figure~\ref{fig:nuscenes-novel-view-example} highlights photometric inconsistencies in baseline
outputs: color and brightness differences between nuScenes source cameras appear as left--right
discontinuities in synthesized frames, resembling the fusion seam of merged-anchor fine-tuning
(Sec.~\ref{sec:ablations}).
Several baselines also distort thin
structures, such as street lamps. SymRegFlow produces more uniform appearance across each
frame and better preserves these structures, consistent with its cross-anchor
consistency regularization.
Appendix~\ref{app:nuscenes-qualitative} provides more comparisons.

\begin{figure}[!htbp]
\centering
\includegraphics[width=0.8\textwidth]{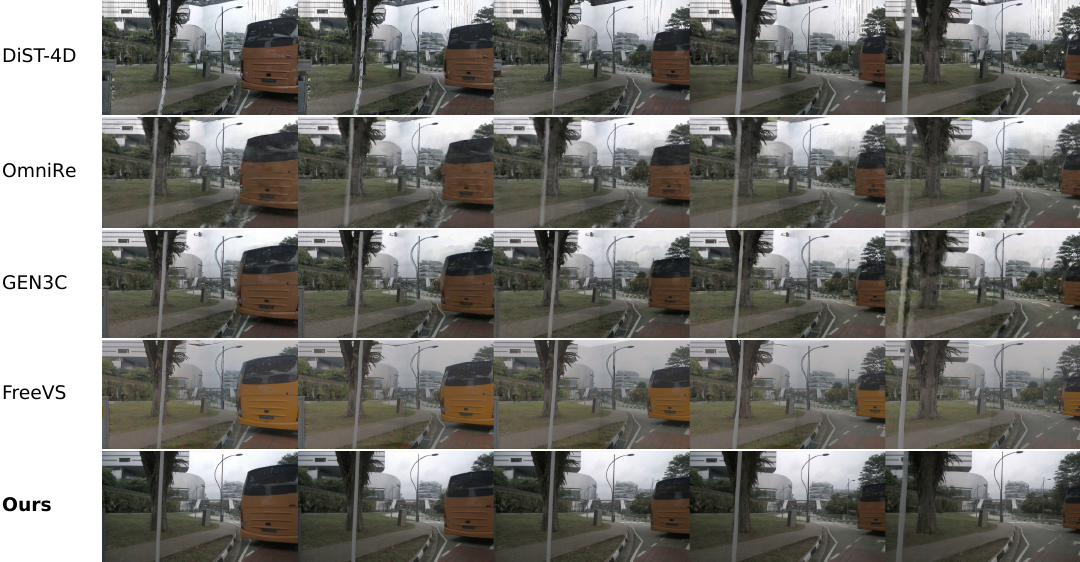}
\caption{Qualitative comparison of novel-view synthesis on nuScenes. The last row shows
\emph{source-cond.} SymRegFlow results for a fair comparison with the baselines.}
\label{fig:nuscenes-novel-view-example}
\end{figure}

\FloatBarrier

\subsection{Ablation Studies}
\label{sec:ablations}
Table~\ref{tab:ablation-quality} ablates anchor strategy, consistency weight, and DCVA for the
midpoint novel pose. Each variant uses one LoRA adapter trained jointly on the same four nuScenes
validation scenes and evaluated at step $2{,}048$; Appendix~\ref{app:experimental-details} gives
details and further analyses.

\paragraph{Model configurations.}
We compare three anchor baselines: $A_L$ and $A_R$ use only the left or right warped anchor,
respectively, while $A_M$
uses a mask-weighted pixel-space merge of both RGB anchors. All three use DCVA for fair comparison.
The SymRegFlow dual-anchor sweep $C_\gamma$ uses DCVA and $\lambda_c=\gamma$ for
$\gamma\in\{0,0.01,0.05,0.1,0.5,1,2\}$. Across these weights, $C_{0.05}$ performs best, with the
lowest existing- and novel-view FVD and MEt3R and the highest VBench subject and background
consistency; we therefore select the optimal $\lambda_c^\star=0.05$.
To investigate the influence of DCVA, $C_{0.05}^{\mathrm{FA}}$ replaces directed attention with full attention at the same weight.

\begin{table}[t]
\centering
\caption{Ablation study on four nuScenes scenes. $\lambda_c$ is the consistency regularization weight
(--: not applicable).
The superscript `FA' in $C_{0.05}^{\mathrm{FA}}$ denotes full attention without DCVA.
VBench scores are multiplied by 100.
\textbf{Bold}: best; \underline{underline}: second-best.}
\label{tab:ablation-quality}
\fontsize{8}{9.5}\selectfont
\setlength{\tabcolsep}{2pt}
\renewcommand{\arraystretch}{1.12}
\begin{tabular*}{\columnwidth}{@{\extracolsep{\fill}}lccccc|ccccc@{}}
\toprule
\multirow{2}{*}{\textbf{Model}}
& \multicolumn{3}{c}{\textbf{Anchors}}
& \multirow{2}{*}{\textbf{DCVA}}
& \multirow{2}{*}{$\lambda_c$}
& \multicolumn{1}{c}{\textbf{Existing view}}
& \multicolumn{4}{c}{\textbf{Novel view}} \\
\cmidrule(lr){2-4}\cmidrule(lr){7-7}\cmidrule(l){8-11}
& Left & Right & Merge & & & FVD$\downarrow$
& FVD$\downarrow$
& VBench Subj.$\uparrow$ & VBench Bg.$\uparrow$
& MEt3R$\downarrow$ \\
\midrule
\multicolumn{11}{@{}l}{\textit{Anchor baselines}} \\
\addlinespace[2pt]
$A_L$ & $\checkmark$ &  &  & $\checkmark$ & -- & 238.14 & 562.23 & 87.86 & 90.36 & 0.2955 \\
$A_R$ &  & $\checkmark$ &  & $\checkmark$ & -- & \textbf{198.86} & 605.63 & 88.79 & 92.20 & 0.3512 \\
$A_M$ &  &  & $\checkmark$ & $\checkmark$ & -- & 213.97 & 704.32 & 87.71 & 91.78 & 0.2787 \\
\midrule
\multicolumn{11}{@{}l}{\textit{SymRegFlow variants (ours)}} \\
\addlinespace[2pt]
$C_0$ & $\checkmark$ & $\checkmark$ &  & $\checkmark$ & 0 & 210.19 & \underline{545.14} & \underline{89.64} & 92.56 & 0.3380 \\
$C_{0.01}$ & $\checkmark$ & $\checkmark$ &  & $\checkmark$ & 0.01 & 233.29 & 558.56 & 89.51 & 91.53 & 0.2832 \\
$C_{0.05}$ & $\checkmark$ & $\checkmark$ &  & $\checkmark$ & 0.05 & \underline{209.37} & \textbf{498.94} & \textbf{89.73} & \textbf{92.75} & \underline{0.2660} \\
$C_{0.1}$ & $\checkmark$ & $\checkmark$ &  & $\checkmark$ & 0.1 & 223.80 & 574.24 & 88.77 & 92.57 & 0.2935 \\
$C_{0.5}$ & $\checkmark$ & $\checkmark$ &  & $\checkmark$ & 0.5 & 212.84 & 575.07 & 88.50 & 92.31 & 0.3032 \\
$C_1$ & $\checkmark$ & $\checkmark$ &  & $\checkmark$ & 1 & 237.80 & 624.15 & 89.09 & 92.32 & 0.3293 \\
$C_2$ & $\checkmark$ & $\checkmark$ &  & $\checkmark$ & 2 & 271.87 & 701.95 & 88.82 & 91.88 & 0.2969 \\
$C_{0.05}^{\mathrm{FA}}$ & $\checkmark$ & $\checkmark$ &  &  & 0.05 & 261.40 & 586.53 & 88.93 & \underline{92.62} & \textbf{0.1996} \\
\bottomrule
\end{tabular*}
\end{table}

\begingroup
\raggedbottom
\paragraph{Effectiveness of dual-anchor fine-tuning.}
Dual-anchor fine-tuning improves novel-view generation even without consistency regularization:
$C_0$ outperforms all three anchor baselines in novel-view FVD and both VBench consistency scores.
For example, $C_0$ lowers novel-view FVD from $562.23$ for the strongest single-anchor baseline
$A_L$ to $545.14$, while $C_{0.05}$ lowers it from $704.32$ for merged-anchor $A_M$ to $498.94$.
Although $A_M$ uses the same two anchors as SymRegFlow for training, it inherits artifacts from
mismatched RGB merging, including misaligned overlaps and residual gray pixels where inpainting
fails (Fig.~\ref{fig:anchor-local-copy}).
$A_R$ nevertheless achieves the best existing-view FVD, plausibly because its single-anchor
objective favors reconstructing existing views without competing consistency regularization; this
advantage does not translate into stronger novel-view generation.

To better expose the biases in $A_L$ and $A_R$, we conduct a controlled one-scene experiment.
Both models tend to copy anchor-visible content when inpainting missing regions, duplicating
construction workers and background buildings and violating camera-pose symmetry. In
Fig.~\ref{fig:anchor-local-copy}, queries on duplicated construction
workers (red) attend most strongly to their source-supported counterparts (cyan), reaching
$13.82\times$ and $12.22\times$ uniform attention, respectively.
Cross-view and within-frame analyses further show that SymRegFlow reduces source-dependent
attention biases in $A_L$ and $A_R$ (Appendix~\ref{app:anchor-attention}).
$A_M$ avoids duplication but retains a conspicuous fusion seam and geometric discontinuity.
SymRegFlow avoids both artifacts.

\begin{figure}[!htbp]
\centering
\includegraphics[width=\textwidth]{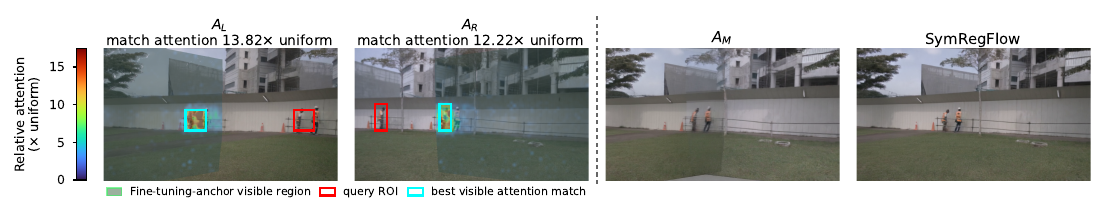}
\caption{Examples of attention bias in anchor baselines. $A_L$ and $A_R$ copy workers
from source-supported regions (cyan boxes) into missing regions (red boxes); green shading indicates
anchor-visible regions. The colorbar shows attention strength as a multiple of uniform attention.}
\label{fig:anchor-local-copy}
\end{figure}

\paragraph{Effectiveness of $\lambda_c$.}
Increasing $\lambda_c$ from $0$ to $0.01$ to $0.05$ reduces MEt3R from $0.3380$ to $0.2832$
to $0.2660$. Novel-view FVD improves from $545.14$ to $498.94$, despite a slight deterioration
at $0.01$; $\lambda_c=0.05$ also maximizes both VBench consistency scores. Larger weights
worsen FVD, from $574.24$ at $0.1$ to $701.95$ at $2$ (Table~\ref{tab:ablation-quality}).
This optimum qualitatively matches Proposition~\ref{prop:clean-reference-improvement},
whose excess risk decreases to $\lambda_c^\star(t)$ and increases thereafter.
Appendix~\ref{app:video-ablation-plots} provides radar profiles.

\paragraph{Effectiveness of DCVA.}
Replacing DCVA with full attention increases existing-view FVD by $24.85\%$ and worsens
novel-view FVD, but lowers MEt3R by $24.96\%$ with similar VBench scores. It also requires
$1.47\times$ the backbone forward/backward time (Appendix~\ref{app:video-ablation-plots}). DCVA thus
trades some cross-view agreement for efficiency and view fidelity.

\FloatBarrier
\endgroup

\section{Related Work}

\subsection{Symmetry and Equivariance in World Models}
Generative models commonly preserve symmetric target distributions through equivariant score or
velocity fields \citep{kohler2020equivariantflows}. Equivariant diffusion uses an
$\mathrm{E}(3)$-equivariant denoiser, while equivariant flow-matching methods constrain continuous
normalizing-flow vector fields or probability paths
\citep{hoogeboom2022edm,klein2023equifm,song2023equifm}. Flow Equivariant World Models instead
build motion-aligned equivariance into a structured latent memory \citep{flowequivwm}. More
generally, these methods require symmetry-specific denoisers, vector fields, or latent-memory
architectures. Architecture-agnostic alternatives avoid such redesign but rely, respectively, on
group-averaged training targets or symmetry-structured noise during training and sampling
\citep{tong2025targetsymm,equivdm}. SymRegFlow introduces a symmetry inductive bias through
its fine-tuning objective, aligning denoising outputs induced by independently warped, noisy anchors
over the union of valid anchor supports while retaining the underlying flow-matching formulation.

\subsection{Camera-Controlled Video World Models}
\paragraph{Training-free camera control.}
Training-free methods keep a pretrained video generator frozen and enforce a target trajectory at
inference through latent replacement, sampling-time guidance, or test-time
optimization \citep{trainingfreecam,zero4d,latentreframe,hcontrol}. They avoid paired training data
but typically require modified sampling or trajectory-specific refinement. SymRegFlow instead learns
camera control by fine-tuning and retains the standard sampler for arbitrary target poses;
\emph{source-cond.} inference additionally uses source-latent replacement.

\paragraph{Training-based camera control.}
MotionCtrl, CameraCtrl, 4DiM, and ReCamMaster amortize camera-controlled generation using costly
pose-annotated or synchronized pose--video pairs, generally producing one target stream
\citep{motionctrl,cameractrl,fourdim,recammaster}. ReCapture \citep{recapture} fits video-specific
LoRAs to valid pixels of a noisy warped anchor for each requested trajectory, masking missing
regions and relying on the pretrained prior to repair artifacts. It neither couples source-camera
anchors nor enforces cross-view symmetry. SymRegFlow instead uses dataset-level fine-tuning and
agreement among independently warped anchors to reduce view-specific errors.

\paragraph{Joint-view world models.}
Joint-view generators couple synchronized spatio-temporal camera tokens through self- or cross-view
attention, conditioned on rig identities, maps, boxes, text, or actions. Examples include MagicDrive,
Drive-WM, and DriveDreamer-2 for driving \citep{magicdrive,drivewm,drivedreamer2}, and
Cosmos-Predict/Transfer2.5 for driving and multi-camera robotics \citep{cosmos25}. Their cross-camera
consistency remains tied to training-rig slots and poses. SymRegFlow introduces a continuously
pose-conditioned target slot while preserving the pretrained multi-view prior, without novel-view
RGB supervision.

\section{Conclusion and Discussion}
SymRegFlow extends pretrained multi-view world models to cross-view-consistent video generation
across continuous viewpoints without novel-view RGB supervision, combining masked dual-anchor
supervision with cross-anchor denoising consistency to mitigate source-dependent biases. Under an
affine Gaussian surrogate, suitable regularization recovers the clean-reference optimum at each
noise level. On nuScenes, SymRegFlow achieves the lowest FVD and FVMD among evaluated baselines,
with source-conditioned inference also attaining the best FID and instance preservation; ablations
confirm improved novel-view quality and reduced source-dependent attention bias.

\paragraph{Limitations.}
SymRegFlow depends on anchor quality: depth, projection, and visibility errors can corrupt
supervision, especially near occlusions or under large viewpoint changes. Dual-anchor consistency
cannot correct shared systematic errors, and jointly unseen regions lack direct supervision. Better
anchors and geometric-uncertainty modeling remain open directions.

\bibliography{references}
\bibliographystyle{conference}

\clearpage
\appendix
\section{Training and Inference Algorithms}
\label{app:algorithms}

Algorithms~\ref{alg:symregflow-training}--\ref{alg:symregflow-generate} use the notation of
Sec.~\ref{sec:equiv} in latent space, with source indices $\mathcal S=\{1,2\}$ and slot order
$(1,\mathrm n,2)$. The shared condition $\mathbf c$ includes text, source-camera identities,
and world-scenario controls encoded in the same slot order. We abbreviate
$\widehat{\mathbf Z}_{0,\theta}^{(s)}$ as $\widehat{\mathbf Z}^{(s)}$.
Losses use elementwise means; for a mask broadcast over latent channels,
\begin{equation}
\|R\|_M^2:=\frac{\sum_i m_i R_i^2}{\max(\sum_i m_i,1)},
\qquad M=\operatorname{diag}(m).
\label{eq:algorithm-masked-loss}
\end{equation}
Thus, an empty mask contributes zero loss.

\begin{algorithm}[H]
\caption{SymRegFlow training}
\label{alg:symregflow-training}
\small
\begin{algorithmic}[1]
\REQUIRE Prepared clips with sources $x^1,x^2$, anchors and masks
$(A^s,m^s)=\Pi_{s\rightarrow\mathrm n}(x^s)$, condition $\mathbf c$, and novel camera
$(\mathbf T^{\mathrm n},K^{\mathrm n})$; pretrained $\theta$; initialized $e_\phi$;
frozen autoencoder $(\mathcal E,\mathcal D)$; weight $\lambda_c$;
training noise law $p_{\mathrm{train}}$ and weighting $\omega(t)$ (see below).
\STATE Freeze base weights; train LoRA and $e_\phi$; enable DCVA in both passes.
\FOR{each optimization step}
    \STATE Sample a prepared clip and its novel camera, anchors, masks, and condition.
    \STATE $X^s\gets\mathcal E(x^s),\quad A^s\gets\mathcal E(A^s),\qquad s\in\mathcal S$
    \STATE $m^s\gets\mathbf 1[\operatorname{Resize}_{\mathrm{latent}}(m^s)>0.5],\quad
        M^s\gets\operatorname{diag}(m^s),\qquad s\in\mathcal S$
    \STATE $M^{\cup}\gets\operatorname{diag}(\max(m^1,m^2))$
    \STATE $\mathbf c^{\mathrm n}\gets(\mathbf c,e_\phi(\mathbf T^{\mathrm n},K^{\mathrm n}))$
    \STATE Sample $t\sim p_{\mathrm{train}}$ and
        $\boldsymbol\epsilon=(\epsilon^1,\epsilon^{\mathrm n},\epsilon^2)\sim\mathcal N(0,I)$ once.
    \FOR{$s\in\mathcal S$}
        \STATE $\mathbf Z_t^{(s)}\gets(1-t)(X^1,A^s,X^2)+t\boldsymbol\epsilon$
        \STATE $\widehat{\mathbf Z}^{(s)}\gets
            f_\theta(\mathbf Z_t^{(s)},\mathbf c^{\mathrm n},t)$
            \hfill\COMMENT{shared noise, time, and condition}
    \ENDFOR
    \STATE $\displaystyle\mathcal L_{\mathrm{src}}\gets
        \frac12\sum_{s\in\mathcal S}\sum_{r\in\mathcal S}
        \big\|\widehat Z^{r,(s)}-X^r\big\|_2^2$
    \STATE $\displaystyle\mathcal L_{\mathrm{anc}}\gets
        \sum_{s\in\mathcal S}\big\|\widehat Z^{\mathrm n,(s)}-A^s\big\|_{M^s}^2$
    \STATE $\displaystyle\mathcal L_{\mathrm{cons}}\gets
        \big\|\widehat Z^{\mathrm n,(1)}-\widehat Z^{\mathrm n,(2)}\big\|_{M^{\cup}}^2$
    \STATE $\mathcal L\gets\omega(t)(\mathcal L_{\mathrm{src}}+\mathcal L_{\mathrm{anc}})
        +\lambda_c\mathcal L_{\mathrm{cons}}$
    \STATE Backpropagate $\mathcal L$ once through both passes; update only
        $(\theta_{\mathrm{LoRA}},\phi)$.
\ENDFOR
\RETURN $(\theta_{\mathrm{LoRA}},\phi)$
\end{algorithmic}
\end{algorithm}

\paragraph{Prediction parameterization.}
Setting $\omega(t)=t^{-2}$ gives Eq.~(\ref{eq:symregflow-objective}).
The video implementation retains the pretrained velocity parameterization
$\mathbf u_\theta$, using
$f_\theta(\mathbf Z_t,\mathbf c^{\mathrm n},t):=\mathbf Z_t-t\mathbf u_\theta(\mathbf Z_t,\mathbf c^{\mathrm n},t)$
for its clean prediction. Its supervised velocity losses are equivalent to
$\omega(t)=t^{-2}$, since
$\|\mathbf u_\theta-(\boldsymbol\epsilon-\mathbf X)\|^2
=t^{-2}\|\widehat{\mathbf Z}_{0,\theta}-\mathbf X\|^2$.
Consistency is computed on clean predictions without this factor.
Here $p_{\mathrm{train}}$ is the training scheduler's distribution of normalized noise
levels $t\in(0,1)$; mask resizing is trilinear. Both passes reuse the same conditioning dropout.

\paragraph{DCVA execution.}
Source queries attend only to source keys; novel queries attend to all three views.
The four-stream implementation used in the DCVA experiments shares source computation and
isolates the two novel streams, giving the same two-pass objective and gradients.

\clearpage
\paragraph{Sampling conventions.}
Both modes use fine-tuned SymRegFlow models with DCVA, without anchors or geometric warping.
Let $t_0>\cdots>t_N=0$ be the scheduler's decreasing noise levels.
$\mathbf u_\theta^{\mathrm{CFG}}$ denotes classifier-free guided velocity; the backbone's
guidance convention is $\mathbf u_{\mathrm{cond}}+w(\mathbf u_{\mathrm{cond}}-\mathbf u_{\mathrm{uncond}})$.
$\operatorname{SchedulerStep}$ advances the joint latent state and its solver history using the
pretrained numerical solver (UniPC in the video implementation).

\begin{algorithm}[H]
\caption{\emph{Source-cond.} inference: novel-view synthesis conditioned on observed sources}
\label{alg:symregflow-nvs}
\small
\begin{algorithmic}[1]
\REQUIRE Fine-tuned $(\theta_{\mathrm{LoRA}},\phi)$; source videos $x^1,x^2$;
condition $\mathbf c$; novel camera $(\mathbf T^{\mathrm n},K^{\mathrm n})$;
noise schedule $\{t_k\}_{k=0}^{N}$; guidance $w$;
replacement horizon $N_{\mathrm{rep}}\in\{0,\ldots,N\}$.
\STATE $X^s\gets\mathcal E(x^s),\quad s\in\mathcal S;\qquad
    \mathbf c^{\mathrm n}\gets(\mathbf c,e_\phi(\mathbf T^{\mathrm n},K^{\mathrm n}))$
\STATE Sample $\boldsymbol\epsilon=(\epsilon^1,\epsilon^{\mathrm n},\epsilon^2)
    \sim\mathcal N(0,I)$; initialize $\mathbf Z_{t_0}\gets\boldsymbol\epsilon$ and the scheduler.
\FOR{$k=0,\ldots,N-1$}
    \STATE $\mathbf Z_{\mathrm{in}}\gets\mathbf Z_{t_k}$
    \IF{$k<N_{\mathrm{rep}}$}
        \STATE $\widetilde Z^s_{t_k}\gets(1-t_k)X^s+t_k\epsilon^s,\qquad s\in\mathcal S$
        \STATE $\mathbf Z_{\mathrm{in}}\gets
            (\widetilde Z^1_{t_k},Z^{\mathrm n}_{t_k},\widetilde Z^2_{t_k})$
            \hfill\COMMENT{replace source inputs only}
    \ENDIF
    \STATE $\widehat{\mathbf u}\gets
        \mathbf u_\theta^{\mathrm{CFG}}(\mathbf Z_{\mathrm{in}},\mathbf c^{\mathrm n},t_k;w)$
    \STATE $\mathbf Z_{t_{k+1}}\gets\operatorname{SchedulerStep}
        (\mathbf Z_{t_k},\widehat{\mathbf u},t_k\rightarrow t_{k+1})$
\ENDFOR
\RETURN $\mathcal D(Z^{\mathrm n}_{t_N})$
\end{algorithmic}
\end{algorithm}

The source noise $\epsilon^s$ is fixed throughout sampling. The horizon $N_{\mathrm{rep}}$
counts initial replacement steps: $N_{\mathrm{rep}}=0$ disables replacement, while
$N_{\mathrm{rep}}=N$ replaces source inputs at every step. Substitution changes only the denoiser
input; the scheduler advances its maintained state $\mathbf Z_{t_k}$, which supplies all slots
once $k\geq N_{\mathrm{rep}}$. The novel slot follows a single sampling trajectory throughout.
The \emph{source-cond.} results in Table~\ref{tab:quantitative_comparison} use $N_{\mathrm{rep}}=30$
and $N=35$ (Sec.~\ref{app:ablation-training}).

\begin{algorithm}[H]
\caption{\emph{Joint-generated} inference: joint three-view generation from noise}
\label{alg:symregflow-generate}
\small
\begin{algorithmic}[1]
\REQUIRE Fine-tuned $(\theta_{\mathrm{LoRA}},\phi)$; condition $\mathbf c$;
novel camera $(\mathbf T^{\mathrm n},K^{\mathrm n})$;
noise schedule $\{t_k\}_{k=0}^{N}$; guidance $w$.
\STATE $\mathbf c^{\mathrm n}\gets(\mathbf c,e_\phi(\mathbf T^{\mathrm n},K^{\mathrm n}))$
\STATE Initialize $\mathbf Z_{t_0}=(Z^1_{t_0},Z^{\mathrm n}_{t_0},Z^2_{t_0})
    \sim\mathcal N(0,I)$ and the scheduler.
\FOR{$k=0,\ldots,N-1$}
    \STATE $\widehat{\mathbf u}\gets
        \mathbf u_\theta^{\mathrm{CFG}}(\mathbf Z_{t_k},\mathbf c^{\mathrm n},t_k;w)$
    \STATE $\mathbf Z_{t_{k+1}}\gets\operatorname{SchedulerStep}
        (\mathbf Z_{t_k},\widehat{\mathbf u},t_k\rightarrow t_{k+1})$
\ENDFOR
\RETURN $(\mathcal D(Z^1_{t_N}),\mathcal D(Z^{\mathrm n}_{t_N}),\mathcal D(Z^2_{t_N}))$
\end{algorithmic}
\end{algorithm}

\emph{Joint-generated} inference uses no observed RGB inputs or source-slot substitution.
Table~\ref{tab:quantitative_comparison} reports both inference modes.

\clearpage

\clearpage
\section{Proof and Numerical Illustration of Proposition~\ref{prop:clean-reference-improvement}}
\label{app:prop1-proof}

We analyze the scalar, fully visible reduction stated in
Proposition~\ref{prop:clean-reference-improvement}, conditional on a fixed
$\mathbf c^{\mathrm n}$ and $t\in(0,1)$. The source-view predictions are held fixed,
so their loss is constant with respect to the scalar novel-view parameter.
With unit masks, the union and intersection coincide; this analysis does not compare the two
masking choices.
All expectations below are over the clean sample, warping errors, and diffusion
noise, or their relevant marginals. Write
\begin{equation}
(\sigma^s)^2=\operatorname{Var}(E^s)>0,\qquad
\nu=\frac{(\sigma^1)^2+(\sigma^2)^2}{2},\qquad
D(t)=(1-t)^2q+t^2=\operatorname{Var}(Z_t).
\label{eq:prop1-variances}
\end{equation}
The two training inputs are $Z_t^{(s)}=(1-t)A^s+t\epsilon$, with the same
$\epsilon$ in both branches. Removing the constant source term from
Eq.~(\ref{eq:symregflow-objective}) gives
\begin{equation}
\mathcal J_{\lambda_c,t}(\theta)
=\frac{1}{t^2}\sum_{s=1}^2
\mathbb E\!\left[\bigl(f_\theta(Z_t^{(s)},t)-A^s\bigr)^2\right]
+\lambda_c\mathbb E\!\left[
\bigl(f_\theta(Z_t^{(1)},t)-f_\theta(Z_t^{(2)},t)\bigr)^2\right].
\label{eq:prop1-training-objective}
\end{equation}
This retains both the sum of the two anchor losses and the $t^{-2}$ weighting
of the paper's objective; consistency is in clean-prediction space.

\begin{proof}[Proof of Proposition~\ref{prop:clean-reference-improvement}]
First, the chosen affine intercept entails no restriction within the affine
class. For a general predictor $\theta z+b$, each supervised residual has mean
$b-[1-(1-t)\theta]\mu$. The intercept cancels from consistency, so minimizing
over $b$ gives $b=[1-(1-t)\theta]\mu$ and hence
$f_\theta(z,t)=\mu+\theta[z-(1-t)\mu]$.

For the clean reference path, joint Gaussianity gives the optimal clean
Gaussian denoiser and its conditional variance:
\begin{equation}
\mathbb E[x^{\mathrm n}\mid Z_t=z]
=\mu+\theta_{\mathrm{cl}}[z-(1-t)\mu],\qquad
\theta_{\mathrm{cl}}=\frac{(1-t)q}{D(t)},\qquad
\operatorname{Var}(x^{\mathrm n}\mid Z_t)=\frac{qt^2}{D(t)}.
\label{eq:prop1-clean-optimum}
\end{equation}
Moreover, the target velocity is $\epsilon-x^{\mathrm n}$, and
$(Z_t-f_\theta(Z_t,t))/t-(\epsilon-x^{\mathrm n})
=(x^{\mathrm n}-f_\theta(Z_t,t))/t$.
Thus the weighted clean-prediction error is exactly the velocity-matching
risk. Completing its square yields
\begin{equation}
\frac{1}{t^2}\mathbb E_{x^{\mathrm n},\epsilon}
\!\left[\bigl(f_\theta(Z_t,t)-x^{\mathrm n}\bigr)^2\right]
=\frac{D(t)\theta^2-2(1-t)q\theta+q}{t^2}
=\frac{q}{D(t)}+\frac{D(t)}{t^2}(\theta-\theta_{\mathrm{cl}})^2.
\label{eq:prop1-clean-decomposition}
\end{equation}
In particular, $\mathcal R_t^{\mathrm{exc}}(\theta)
=D(t)(\theta-\theta_{\mathrm{cl}})^2/t^2$.

For the training objective, independence and centering of the noises give
\begin{align}
\sum_{s=1}^2\mathbb E\!\left[
\bigl(f_\theta(Z_t^{(s)},t)-A^s\bigr)^2\right]
&=2\left[\bigl((1-t)^2(q+\nu)+t^2\bigr)\theta^2
-2(1-t)(q+\nu)\theta+q+\nu\right],\nonumber\\
\mathbb E\!\left[
\bigl(f_\theta(Z_t^{(1)},t)-f_\theta(Z_t^{(2)},t)\bigr)^2\right]
&=2(1-t)^2\nu\theta^2.
\label{eq:prop1-training-moments}
\end{align}
The second equality uses $Z_t^{(1)}-Z_t^{(2)}=(1-t)(E^1-E^2)$: the shared
clean sample and diffusion noise cancel. Consequently, defining
\begin{equation}
B_{\lambda_c}(t)=(1-t)^2(q+\nu)+t^2+\lambda_c(1-t)^2\nu t^2,
\label{eq:prop1-training-curvature}
\end{equation}
we obtain
\begin{equation}
\mathcal J_{\lambda_c,t}(\theta)
=\frac{2}{t^2}\left[B_{\lambda_c}(t)\theta^2
-2(1-t)(q+\nu)\theta+q+\nu\right],\qquad
\theta_{\lambda_c}=\frac{(1-t)(q+\nu)}{B_{\lambda_c}(t)}.
\label{eq:prop1-regularized-optimum}
\end{equation}
The minimizer is unique since the curvature is $4B_{\lambda_c}(t)/t^2>0$
for $\lambda_c\geq0$. Substituting it into the clean excess risk gives
\begin{equation}
\mathcal R_t^{\mathrm{exc}}(\theta_{\lambda_c})
=\frac{(1-t)^2\nu^2t^2\bigl[1-\lambda_c q(1-t)^2\bigr]^2}
{D(t)B_{\lambda_c}(t)^2}.
\label{eq:prop1-regularized-risk}
\end{equation}
At $\lambda_c=\lambda_c^\star(t)=1/[q(1-t)^2]$,
$B_{\lambda_c^\star}(t)=(q+\nu)D(t)/q$, so
$\theta_{\lambda_c^\star(t)}=\theta_{\mathrm{cl}}$ and the excess risk is zero.
Since $B_{\lambda_c}(t)$ strictly increases with $\lambda_c$, the optimizer
$\theta_{\lambda_c}$ strictly decreases through $\theta_{\mathrm{cl}}$ at
$\lambda_c^\star(t)$. Equation~(\ref{eq:prop1-clean-decomposition}) therefore
implies that excess risk strictly decreases before this strength and strictly
increases after it.

For comparison, let $A^{\mathrm M}=\pi A^1+(1-\pi)A^2$ for fixed
$\pi\in[0,1]$, and denote the three anchor-error variances by
\begin{equation}
r_1=(\sigma^1)^2,\qquad r_2=(\sigma^2)^2,\qquad
r_{\mathrm M}=\pi^2(\sigma^1)^2+(1-\pi)^2(\sigma^2)^2.
\label{eq:prop1-baseline-variances}
\end{equation}
Each $r_j$ is positive. Minimizing the single-target weighted supervision
for each $j\in\{1,2,\mathrm M\}$ yields
\begin{equation}
\theta_j=\frac{(1-t)(q+r_j)}{(1-t)^2(q+r_j)+t^2},\qquad
\mathcal R_t^{\mathrm{exc}}(\theta_j)
=\frac{(1-t)^2r_j^2t^2}
{D(t)\bigl[(1-t)^2(q+r_j)+t^2\bigr]^2}>0.
\label{eq:prop1-baseline-risk}
\end{equation}
This proves Eq.~(\ref{eq:prop1-strict-improvement}). Finally,
$\mathcal R_t^{\mathrm{exc}}(\theta_{\lambda_c})$ is continuous in
$\lambda_c\geq0$, and the minimum of the three baseline risks is strictly
positive. The strict comparison therefore persists on an open interval
containing $\lambda_c^\star(t)$.
\end{proof}

\paragraph{Interpretation and scope.}
SymRegFlow combines supervision from both noisy anchors with shared-noise
consistency. Supervision alone fits the inflated anchor variance $q+\nu$;
consistency penalizes sensitivity to anchor disagreement and, at the calibrated
strength, recovers the clean-reference optimum. Stronger regularization further
shrinks the denoiser's slope below the clean optimum, increasing excess risk. Zero excess
risk means attaining the minimum total risk $q/D(t)$, not zero reconstruction
error. The comparison concerns clean-reference evaluation, not the numerical
values of different training objectives. It holds for the stated affine,
independent-error surrogate with unit masks and fixed source predictions;
it does not assert the same guarantee for the full neural network or one
constant $\lambda_c$ across all diffusion times.

\paragraph{Numerical illustration.}
To instantiate the viewpoint-dependent error scale from Sec.~\ref{sec:toy-exp},
set $\sigma^s=\kappa(E_b+d^s)$, with $d^1=\alpha$, $d^2=1-\alpha$.
Figure~\ref{fig:prop1-clean-risk} evaluates
Eqs.~(\ref{eq:prop1-regularized-risk}) and~(\ref{eq:prop1-baseline-risk}) using
$\mu=1.2$, $q=1$, $t=0.5$, $E_b=0.4$, $\kappa=0.6$, and
merge weight $\pi=1-\alpha$. At $\alpha=0.5$ the single-anchor risks coincide;
replacing $\alpha$ by $1-\alpha$ exchanges them and preserves both the merged
and SymRegFlow curves. In all three cases $\lambda_c^\star=4$, since the
calibration depends on $q,t$ rather than the anchor-error variances.

\begin{figure}[htbp]
\centering
\includegraphics[width=\textwidth]{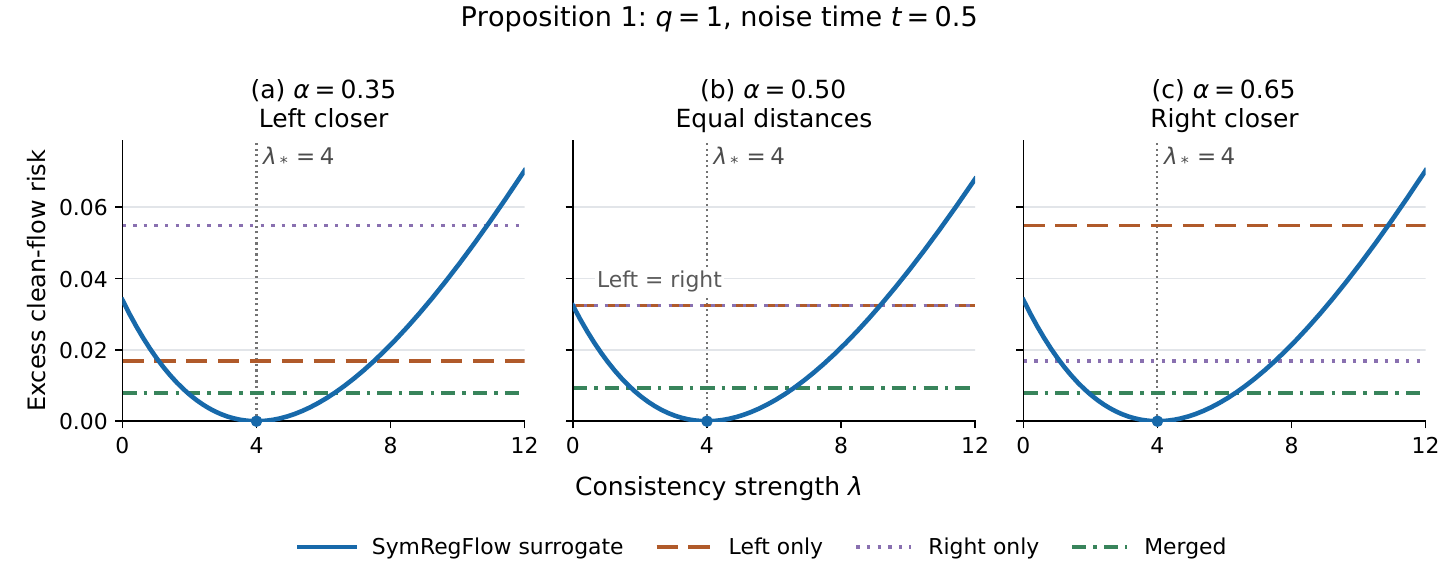}
\caption{Analytic clean-reference excess risk for Proposition~\ref{prop:clean-reference-improvement}
at $\alpha\in\{0.35,0.50,0.65\}$. The horizontal baselines use the optimal
unregularized predictors; the SymRegFlow curve uses the optimal predictor
for each consistency strength ($\lambda$ in the plot denotes $\lambda_c$).
The midpoint makes left/right risks equal, and mirrored viewpoints exchange
them. SymRegFlow has zero excess risk at $\lambda_c^\star=4$ and strictly
improves on all three baselines in a neighborhood of that strength.
These are evaluations of the surrogate formulas, not additional MNIST measurements.}
\label{fig:prop1-clean-risk}
\end{figure}
\FloatBarrier

\section{Implementation Details and Additional Ablation Results}
\label{app:experimental-details}
\label{app:detailed-ablation}

We first describe the shared implementation and study-specific evaluation protocols for the
main video experiments and ablations, then present additional quantitative and attention-bias analyses.

\subsection{Implementation Details}
\label{app:ablation-training}

\paragraph{Backbone and adaptation.}
All video experiments use Cosmos-Transfer2.5 \citep{cosmos25} with 29-frame clips at
$1280\times720$ resolution per view. Fine-tuning uses rank-$96$ LoRA with LoRA alpha $12$,
gray fill for invalid warp regions, and the continuous-view conditioning described below. Dual-anchor variants regularize clean predictions;
the default SymRegFlow configuration uses DCVA and $\lambda_c=0.05$.
For the quantitative nuScenes experiments and ablations, training uses 16 GPUs and a global
batch size of 32. LoRA and view-embedding learning rates are $5\times10^{-4}$ and
$5\times10^{-3}$, respectively, with cosine decay, 25 warm-up steps, and weight decay $0.01$.
Reported quantitative checkpoints are evaluated at step $2{,}048$.

\paragraph{Novel-view pose conditioning.}
Cosmos-Transfer2.5 \citep{cosmos25} conditions each view on a world-scenario video composed
of projected lane/map polylines and 3D object bounding boxes, rather than captured RGB
footage, to specify scene layout and viewpoint. During training, we render this representation
at the novel camera poses and feed its encoded latents into SymRegFlow through the backbone's
control branch. Both Cosmos-Drive-Dreams \citep{Ren2025a} and nuScenes \citep{nuscenes}
provide the map annotations, object geometry, and camera calibration required for this rendering.

To minimize modifications to the pretrained backbone, we retain its integer-indexed embeddings
for existing rig cameras. Our autonomous-driving data synthesis setting requires novel views
interpolated along the six-camera ring. We therefore parameterize the continuous view embedder
by the endpoint camera identities and a scalar ring coordinate, with additional geometric
conditioning supplied by the world-scenario renderings. The pretrained camera embeddings remain
frozen; Fourier features of the global ring coordinate drive a zero-initialized,
endpoint-vanishing MLP residual added to the linearly interpolated embeddings
($\alpha=0.5$ at the midpoint; Sec.~\ref{app:global-view-encoding}).

This design specializes the general pose-conditioned formulation to the interpolation setting
evaluated here. The formulation is compatible with arbitrary target-pose inputs: $e_\phi$ can
instead encode camera extrinsic and intrinsic matrices, Pl\"ucker rays, or other pose
representations and be trained jointly during fine-tuning.

\paragraph{Sampling.}
The nuScenes results in Tables~\ref{tab:quantitative_comparison} and~\ref{tab:ablation-quality}
use 35 denoising steps, guidance scale 3, and seed 42. Table~\ref{tab:quantitative_comparison}
reports both \emph{joint-generated} inference, without observed-RGB source-slot replacement,
and \emph{source-cond.} inference, with replacement during the first 30 steps.
Table~\ref{tab:ablation-quality} uses \emph{joint-generated} inference only.

\paragraph{Distributional metrics.}
Novel-view FID \citep{Heusel2017} and FVD \citep{Unterthiner2018} compare generated clips
with unique captured source-view clips as
\emph{unpaired} distributions; captured midpoint RGB is unavailable, and world-scenario
renderings are never used as RGB ground truth. FID uses 2,048-dimensional Inception features,
and FVD uses 400-dimensional I3D features from 29-frame clips, pooling all scenes in each study.

\paragraph{Cross-view consistency.}
MEt3R \citep{met3r} projects FeatUp-upsampled DINO features through MASt3R geometry and
averages cosine distances over usable overlap and valid synchronized frame pairs. Images
preserve aspect ratio with a maximum long edge of 512 pixels. Valid pairs require at least
5\% raw overlap and 1\% usable overlap after excluding relative depth disagreements above
0.10. The reference views and temporal inputs differ between studies, as specified below.

\subsection{Evaluation Protocols}
\label{app:ablation-protocol}

\paragraph{Thirty-scene comparison.}
\label{app:nuscenes-protocol}
Table~\ref{tab:quantitative_comparison} pools outputs from eight scene-group adapters covering
30 nuScenes scenes, with six midpoint camera pairs per scene and frames 0--115 at 10 fps.
FVD uses 720 nonoverlapping 29-frame clips per distribution, starting at frames 0, 29, 58,
and 87; FID pools their 20,880 frames per distribution. VBench \citep{Huang2024} evaluates
the 720 generated clips independently, using all 29 frames of each clip.
FVMD \citep{fvmd} uses 180 sequences of 116 frames at 10 fps per distribution for
all six methods. Our sequences concatenate four independently generated 29-frame clips;
baselines retain the first 116 frames of each released video. We bilinearly resize frames
to $256\times256$ and use PIPS2 to track 400 points with 16 iterations per 16-frame window.
A stride of one frame yields 18,180 windows per distribution, including windows crossing
generation joins. We concatenate each window's velocity and acceleration histograms and
compute a single Fr\'echet distance between the generated novel-view and unpaired
captured-source feature distributions, pooling all windows across scenes.
Instance preservation (IP) \citep{Zheng2025} uses these same 180
116-frame sequences per model and the official fine-tuned Qwen2.5-VL-3B-Instruct entity-anomaly
detector. Each sequence is split into overlapping 2-second clips with a 1-second stride,
retaining all 20 frames per clip. Following the official duration rule, the ten clips start
at seconds 0--9 and cover frames 0--109, including all three generation joins; the final six
frames are omitted. A sequence scores one only if every clip is classified as free of
unnatural entity merging, duplication, or disappearance, and zero otherwise. We report
100 times the mean over all 180 sequences, with higher scores indicating better preservation.
MEt3R compares novel views with synchronized existing-view videos:
\emph{generated} videos for the \emph{joint-generated} result ($^\dagger$) and \emph{captured} videos
for all other rows. These settings measure internal cross-view consistency and agreement with
observations, respectively. Each uses 41,760 candidate frame pairs per model, averaging its
valid subset.

\paragraph{Four-scene ablation.}
Each variant in Table~\ref{tab:ablation-quality} uses one adapter trained jointly on
scene-0012, scene-0013, scene-0014, and scene-0015 with a $100{,}000$-step optimization
schedule and evaluated at step $2{,}048$. Evaluation uses the same scenes and four segment
starts, yielding 96 novel clips (2,784 frames) and 192 generated endpoint clips per model.
Novel-view FID/FVD use 96 unique captured source clips; existing-view metrics use paired
captured endpoints. VBench uses the 96 separate 29-frame clips, rather than 116-frame sequences.
MEt3R compares novel views with \emph{generated} endpoints over 192 clip pairs
and 5,568 synchronized frame pairs.

\subsection{Additional Ablation Results}
\label{app:video-ablation-plots}

\paragraph{Variants and checkpoint selection.}
$A_L$ and $A_R$ use individual source anchors; $A_M$ merges both RGB anchors using their
visibility masks. All three use DCVA. The dual-anchor variants $C_\gamma$ use four-stream
DCVA with $\gamma\in\{0,0.01,0.05,0.1,0.5,1,2\}$. At $\lambda_c=0.05$,
$C_{0.05}^{\mathrm{FA}}$ replaces DCVA with full attention, changing both attention and
training execution. All 11 reported variants are evaluated at step $2{,}048$; the configured
$C_{0.001}$ run lacks outputs at this checkpoint and is omitted. The following analyses
use the same four-scene measurements as Table~\ref{tab:ablation-quality}.

\paragraph{Hybrid radar summary.}
Figure~\ref{fig:lambda-ablation-radar} summarizes the five metrics in Table~\ref{tab:ablation-quality}:
existing-view FVD; novel-view FVD, VBench subject and background consistency, and MEt3R.
Ranks and min--max bounds use all $N=11$ table variants, while the display contains the seven
$C_\gamma$ models in the consistency-weight sweep.
For direction-aware rank $r$, $S_{\mathrm{rank}}=100(N+1-r)/N$; if $S_{\mathrm{mm}}$ is the direction-aware min--max score
on a 0--100 scale, the displayed radius is $S_{\mathrm{hyb}}=(S_{\mathrm{rank}}+S_{\mathrm{mm}})/2$. $C_{0.05}$ has the largest mean
profile in the sweep, 92.4, followed by $C_0$ at 71.3 and $C_{0.1}$ at 61.2. This descriptive
mean weights the five spokes equally and is not polygon area or a replacement for the individual metrics.

\begin{figure}[htbp]
\centering
\includegraphics[width=\textwidth]{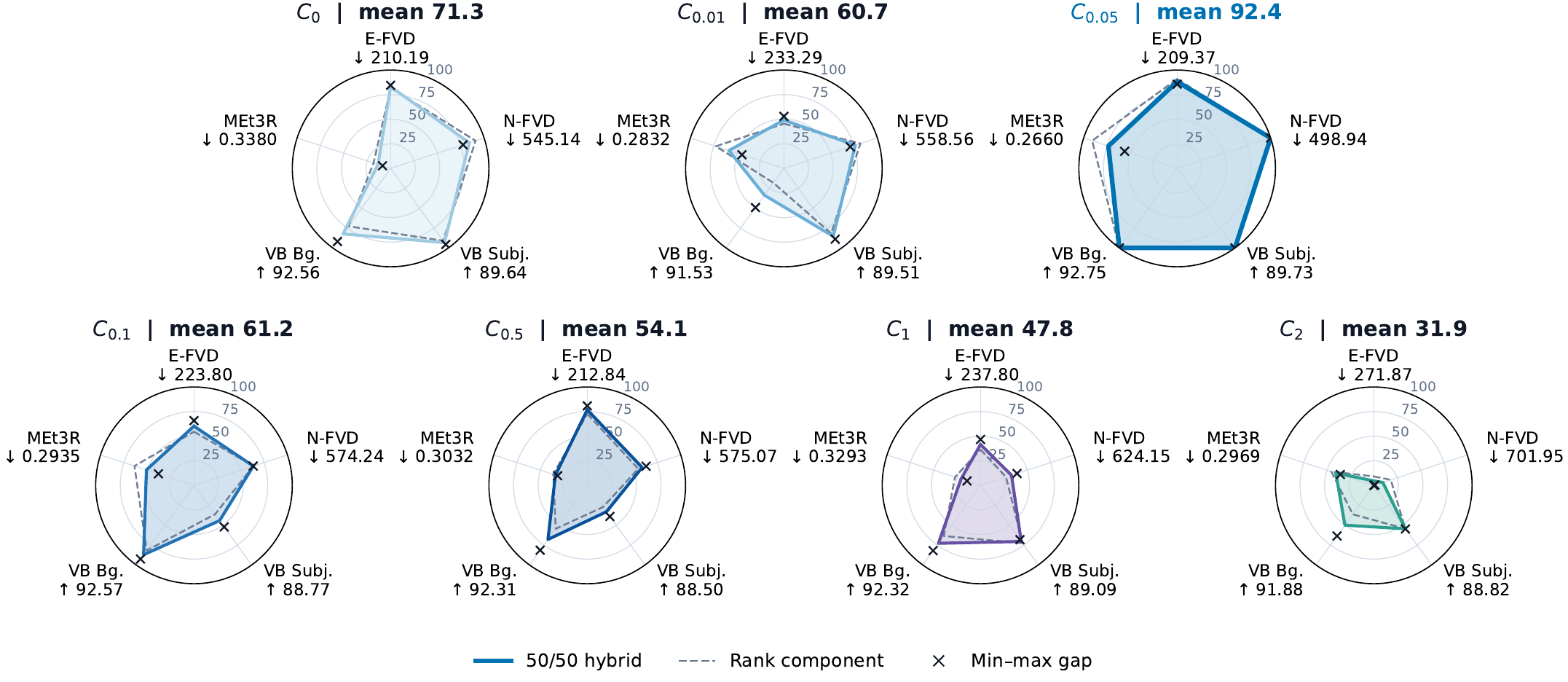}
\caption{Consistency-weight profiles using the five metrics in Table~\ref{tab:ablation-quality},
normalized against all 11 variants at iteration 2,048. Outward is better. Solid polygons show the
50/50 rank--min--max hybrid, dashed polygons its rank component, and crosses its min--max component.
E/N denote existing/novel views; VB denotes VBench. Spoke values follow the table's precision and scaling
(VBench $\times100$); the mean is the arithmetic average of the five normalized spokes.}
\label{fig:lambda-ablation-radar}
\end{figure}

\FloatBarrier

\paragraph{Computational-cost probe.}
\label{app:dcva-cost}
On one NVIDIA RTX 5880 Ada GPU, we compare backbone costs using 28 DiT blocks, width
2,048, 16 heads, rank-96 LoRA, BF16, batch size two, and activation checkpointing at every
block. Randomly initialized models receive $8\times16\times16$ tokens per view and retain
both anchor branches until backward. DCVA shares source-view computation, whereas full
attention recomputes it for each anchor. After one warm-up iteration, medians over three
forward/backward runs are $3.07$ and $4.51$ seconds, respectively ($1.47\times$).
This reduced-resolution probe excludes the control network, VAE, rotary embeddings, data
loading, optimizer, and distributed communication; it is not an end-to-end training benchmark.

\subsection{Attention-Bias Analysis}
\label{app:anchor-attention}

In a controlled one-scene diagnostic, each anchor strategy is trained for 300 iterations and
sampled with identical noise, prompt, and world-state controls. These measurements characterize
attention routing and are distinct from the four-scene benchmark in Table~\ref{tab:ablation-quality}.

\paragraph{Cross-view routing.}
We measure the absolute route gap
$G_{\mathrm{xv}}=|\bar\ell_{\mathrm{n}\rightarrow\mathrm{l}}-
\bar\ell_{\mathrm{n}\rightarrow\mathrm{r}}|$, where $\bar\ell$ is the mean post-RoPE,
pre-softmax QK logit over the captured denoising steps, DiT blocks, heads, and width, height, and
time offset profiles. Table~\ref{tab:anchor-attention-bias} reports gaps of $0.729$ and $0.865$
for $A_L$ and $A_R$, respectively, with preferences aligned to their fine-tuning anchors.
SymRegFlow reduces the gap to $0.307$, a reduction of $57.8\%$ and $64.5\%$, respectively.

\begin{table}[H]
\centering
\caption{Attention bias in the controlled one-scene overfitting experiment. The cross-view columns
report the absolute novel-to-left/right QK route gap $G_{\mathrm{xv}}$ (lower is more symmetric),
averaged over three relative-offset axes. The within-frame columns report invisible-to-visible
attention mass $P_{\mathrm{U}\rightarrow\mathrm{V}}$ under an identical paired visibility mask;
$\Delta$ and the relative excess are single-anchor minus SymRegFlow.}
\label{tab:anchor-attention-bias}
\footnotesize
\setlength{\tabcolsep}{4.2pt}
\renewcommand{\arraystretch}{1.08}
\begin{tabular}{@{}lccc cccc@{}}
\toprule
\multirow{2}{*}{\textbf{Paired anchor}}
& \multicolumn{3}{c}{\textbf{Cross-view route gap} $G_{\mathrm{xv}}$}
& \multicolumn{4}{c}{\textbf{Within-frame mass} $P_{\mathrm{U}\rightarrow\mathrm{V}}$} \\
\cmidrule(lr){2-4}\cmidrule(l){5-8}
& \textbf{Single} & \textbf{SymRegFlow} & \textbf{Reduction}
& \textbf{Single} & \textbf{SymRegFlow} & $\boldsymbol{\Delta}$ & \textbf{Relative excess} \\
\midrule
Left  & 0.729 & 0.307 & 57.8\% & 0.1745 & 0.1447 & +0.0298 & +20.6\% \\
Right & 0.865 & 0.307 & 64.5\% & 0.2091 & 0.1877 & +0.0214 & +11.4\% \\
\bottomrule
\end{tabular}
\end{table}

Figure~\ref{fig:anchor-offsets} retains the relative key--query offset to show where the imbalance
arises, reporting post-RoPE, pre-softmax QK compatibility along the latent width, height, and time axes. The
left-only model raises the novel-to-left profile relative to the right-only model, while the
right-only model raises the novel-to-right profile. SymRegFlow remains substantially closer across
the two routes. These profiles average the same five denoising steps, all 28 DiT blocks, and all
attention heads used in Table~\ref{tab:anchor-attention-bias}.

\begin{figure}[H]
\centering
\includegraphics[width=\textwidth]{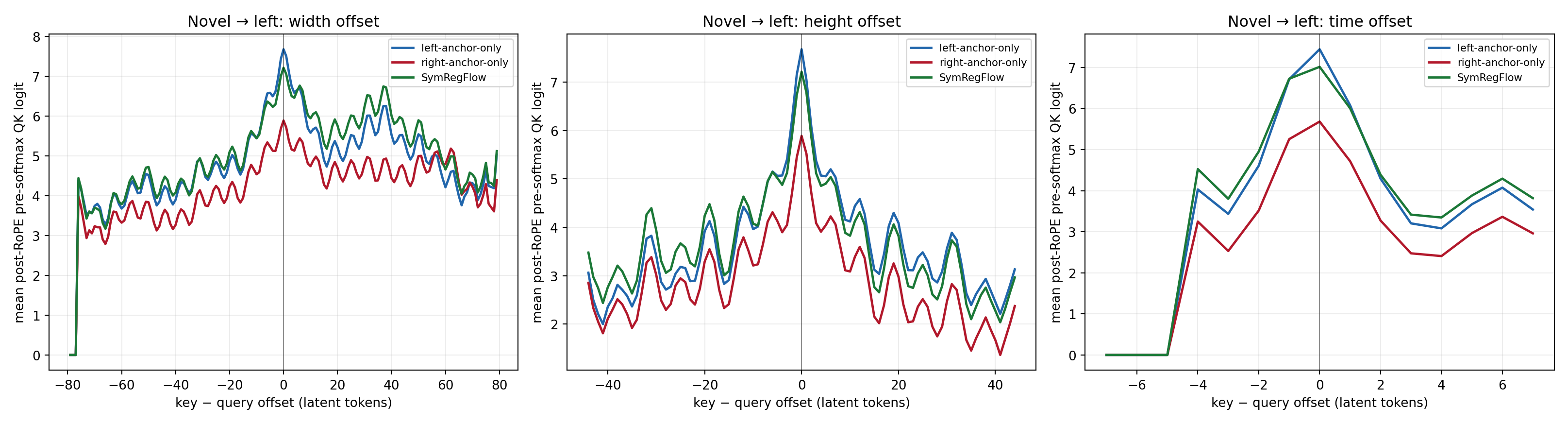}
\vspace{-3pt}
\includegraphics[width=\textwidth]{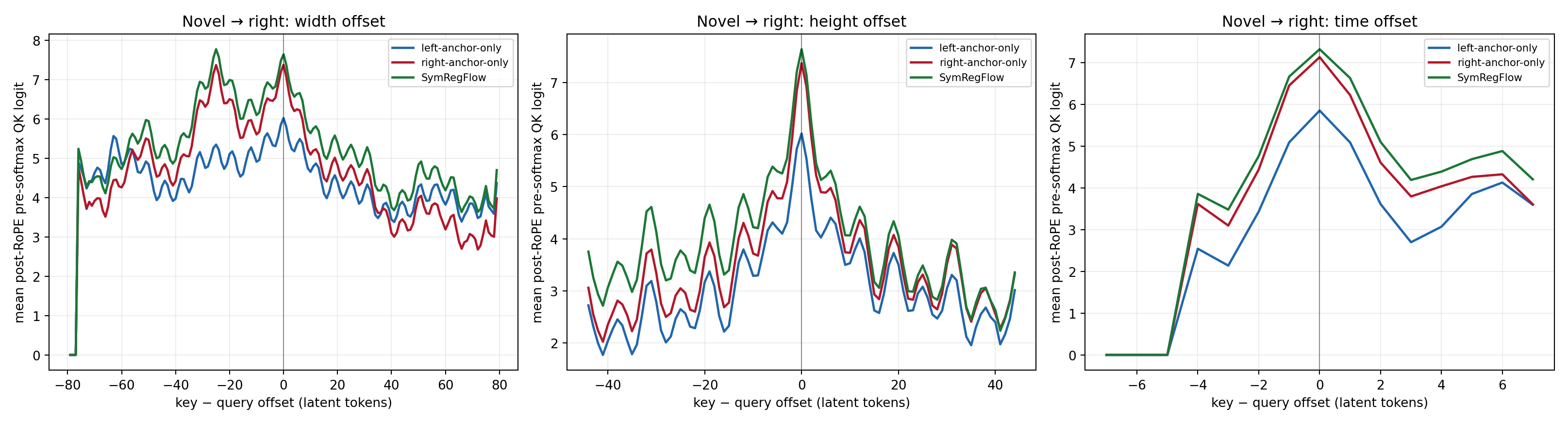}
\caption{Cross-view relative-offset QK profiles for novel-to-left (top) and novel-to-right
(bottom) routes. Curves are count-weighted means over valid token pairs and are further averaged
over captured denoising steps, DiT blocks, and heads. Single-anchor adaptation preferentially
raises compatibility with its fine-tuning source; SymRegFlow reduces this source-dependent route
imbalance.}
\label{fig:anchor-offsets}
\end{figure}

Figure~\ref{fig:anchor-denoising} retains the denoising step at zero spatial offset. The signed
gap reverses with the selected single anchor and remains source-aligned throughout denoising,
whereas SymRegFlow stays much closer to zero. Because these measurements are attention
diagnostics rather than controlled interventions on Q/K, they support a routing-bias explanation
but are not by themselves a causal proof of the repeated-content artifact.

\begin{figure}[H]
\centering
\includegraphics[width=\textwidth]{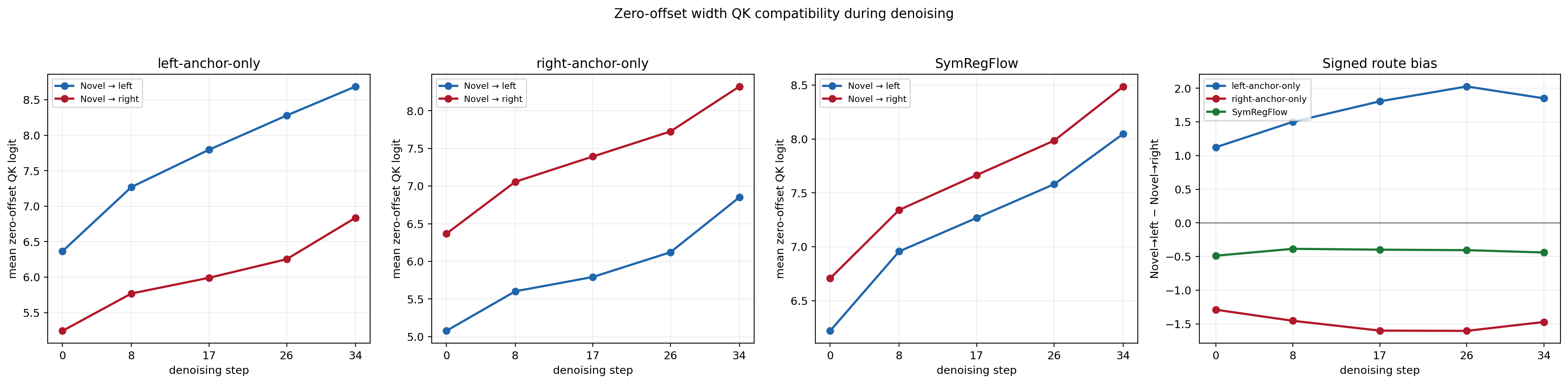}
\caption{Zero-offset cross-view QK compatibility over denoising. The rightmost panel reports the
signed route bias, novel-to-left minus novel-to-right. Its sign follows the fine-tuning anchor for
both single-anchor models, while SymRegFlow has a smaller, stable residual imbalance.}
\label{fig:anchor-denoising}
\end{figure}

\paragraph{Within-frame routing.}
Within the novel-view frame, we measure the same-frame softmax mass
$P_{\mathrm{U}\rightarrow\mathrm{V}}$ from all unambiguously invisible (inpainted) query tokens
$\mathrm U$ to anchor-visible keys $\mathrm V$. Each comparison applies the same left- or
right-anchor visibility mask to the single-anchor model and SymRegFlow. As reported in
Table~\ref{tab:anchor-attention-bias}, $A_L$ and $A_R$ assign $20.6\%$ and $11.4\%$ more mass
to the corresponding visible regions, respectively. Thus, the cross-view preference persists as
biased routing from generated content toward source-supported content inside the novel view.

\paragraph{Local repeated-content diagnostic.}
Figure~\ref{fig:anchor-local-copy} uses generated frame 28, with both attention panels measured at
denoising step 17 and DiT block 20. Red boxes mark queries in invisible (inpainted) regions; cyan
boxes mark the strongest same-size, non-overlapping attention match restricted to the fine-tuning
anchor's visible region (green). Heat maps show same-frame attention relative to uniform attention,
only on visible support. $A_M$ and SymRegFlow are shown as clean RGB controls without attention
measurements; the dashed line separates them from the two attention-analyzed single-anchor panels.

\FloatBarrier

\subsection{Global Continuous Novel-View Encoding}
\label{app:global-view-encoding}

Let $\mathbf e^k\in\mathbb R^D$ denote the frozen pretrained embedding of rig camera $k$, and assign
each camera a fixed circular coordinate $\xi^k\in[0,1)$. For the six-camera rig, front, front-left,
rear-left, rear, rear-right, and front-right are assigned coordinates $0,\frac16,\ldots,\frac56$,
respectively. For adjacent cameras $a$ and $b$ and
interpolation factor $\alpha\in[0,1]$, we define the shortest signed displacement and global novel-view
coordinate as
\begin{align}
    \Delta^{a,b}&=\operatorname{wrap}(\xi^b-\xi^a),\nonumber\\
    \xi^{a,b}(\alpha)&=\bigl(\xi^a+\alpha\Delta^{a,b}\bigr)\bmod 1,
\end{align}
where $\operatorname{wrap}(x)=((x+\tfrac12)\bmod 1)-\tfrac12$. Thus, the interpolation factor is
disambiguated by its position on the rig rather than represented in a pair-local space. We encode this circular coordinate
using seam-continuous fixed harmonics,
\begin{equation}
    \psi(\xi)=\bigl[\cos(2\pi k\xi),\sin(2\pi k\xi)\bigr]_{k=1}^{N_{\mathrm{harm}}},
\end{equation}
with $N_{\mathrm{harm}}=4$ in our implementation.

The continuous novel-view embedding is anchored to the pretrained camera-embedding space:
\begin{equation}
    \mathbf e^{a,b}(\alpha)
    =(1-\alpha)\mathbf e^a+\alpha\mathbf e^b
    +4\alpha(1-\alpha)r_\phi\!\left(\psi(\xi^{a,b}(\alpha))\right),
    \label{eq:global-novel-view-embedding}
\end{equation}
where $r_\phi:\mathbb R^{2N_{\mathrm{harm}}}\rightarrow\mathbb R^D$ is a small MLP whose output layer is initialized
to zero. The first two terms preserve the geometry of the pretrained embedding space, while the learned
residual models smooth deviations along the continuous camera ring. The envelope $4\alpha(1-\alpha)$
vanishes at both endpoints, guaranteeing $\mathbf e^{a,b}(0)=\mathbf e^a$ and
$\mathbf e^{a,b}(1)=\mathbf e^b$.

\section{Additional $\mathrm{SE}(2)$ Toy Study}
\label{app:toy}

\paragraph{Implementation.}
The base joint image transformer is trained for $8{,}000$ iterations on paired MNIST views at
$-90^\circ$ and $0^\circ$, using a batch size of $512$ and learning rate $2\times10^{-4}$.
Following JiT \citep{li2025backtobasics}, the toy implementation replaces $t$ by $\max(t,0.05)$
when converting clean predictions or targets to velocities, giving the supervised weight
$1/\max(t,0.05)^2$.
Every continuous-pose variant initializes from this same checkpoint and is fine-tuned in full for
$2{,}000$ iterations, with batch size $256$ and learning rate $2\times10^{-4}$. The target angle
is sampled uniformly through $\alpha\sim\mathcal U(0,1)$ and supplied through the learned scalar
pose embedding $e_\phi(\alpha)$. All variants use the same optimizer settings, anchor generator, and
source-preservation term. Denoising-output SymRegFlow uses Directed Cross-View
Attention (DCVA), as defined in Sec.~\ref{sec:equiv}, and evaluates its two anchor-conditioned
branches through the exactly equivalent four-stream factorization
$[x^1,A^1,A^2,x^2]$. We sweep
$\lambda_c\in\{0,0.1,0.25,0.5,1,2,4\}$ and select $\lambda_c=2$ by the mean
exact-target FD over the $45$ pose--distortion cells.

\paragraph{Baseline construction.}
For each sampled pose, SymRegFlow uses the two pseudo-examples
$\widetilde{\mathbf x}^{(s)}=(x^1,A^s,x^2)$ for the paired forward passes in Eq.~(\ref{eq:paired-forward}).
The $A^1$-only and $A^2$-only baselines each use the corresponding single pseudo-example;
Merged anchor uses $(x^1,A^{\mathrm M},x^2)$, where $A^{\mathrm M}$ is the mask-weighted fusion
of both anchors.

\begin{figure*}[t]
\centering
\includegraphics[width=\textwidth]{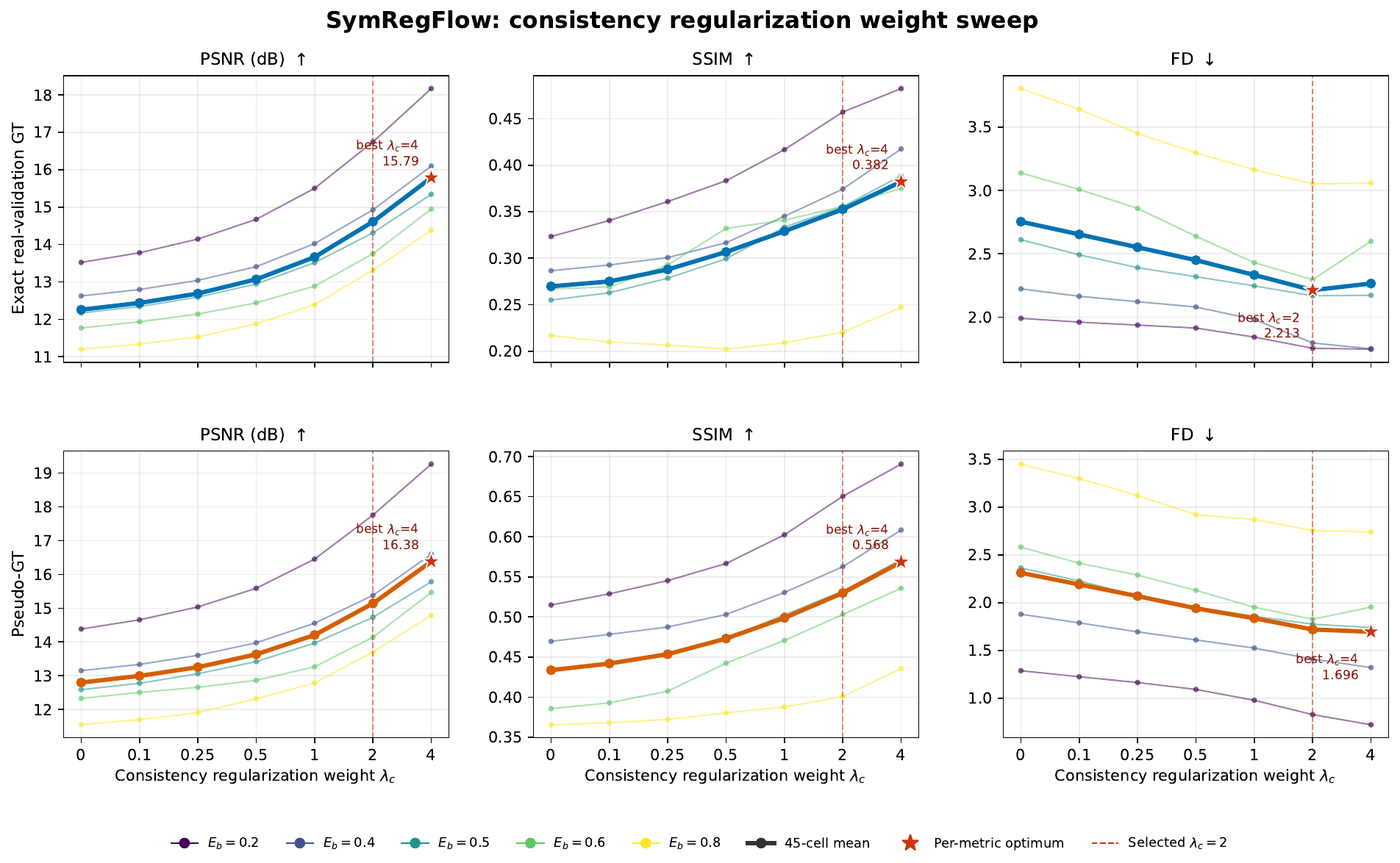}
\caption{SymRegFlow consistency regularization weight sweep against exact held-out rotations (top) and pseudo
targets (bottom). The dashed line indicates the selected $\lambda_c=2$; stars mark the best sampled
weight for each metric. FD denotes pixel-space Fr\'echet distance on $16\times16$ pixel features.}
\label{fig:toy-lambda-sweep}
\end{figure*}

\paragraph{Consistency regularization weight sensitivity.}
Figure~\ref{fig:toy-lambda-sweep} reports the complete sweep. Each thin curve averages the nine
interior target poses at one fixed distortion base $E_b$; each thick curve averages all five bases,
and therefore all $45$ pose--distortion cells. On exact targets, PSNR and SSIM continue increasing
through $\lambda_c=4$, reaching $15.79$ and $0.382$, respectively. Exact-target FD instead falls
from $2.754$ at $\lambda_c=0$ to its minimum of $2.213$ at $\lambda_c=2$, then rises to $2.265$ at
$\lambda_c=4$. We therefore select $\lambda_c=2$ using the exact-target FD criterion. All three
pseudo-target metrics continue improving through $\lambda_c=4$, consistent with stronger coupling
fitting the imperfect anchors more closely; their divergence from exact-target FD motivates
selection on withheld exact rotations.

\paragraph{Anchor corruption.}
The angular distances from $x^1$ and $x^2$ to the target at
$\vartheta_\alpha=-90^\circ(1-\alpha)$ are $90^\circ\alpha$ and $90^\circ(1-\alpha)$,
respectively. Normalizing by $90^\circ$ gives $d^1(\alpha)=\alpha$ and
$d^2(\alpha)=1-\alpha$. Let $E_b$ denote the base distortion level.
For $s\in\mathcal S=\{1,2\}$, we set the scale of $E^s$
to $E_b+d^s(\alpha)$, confidence to $w^s=1-d^s(\alpha)$, and loss mask to
$m^s=w^s\bar m^s$, where $\bar m^s$ is the binary sampling-validity mask. The warp perturbs
morphology, anisotropic scale and shear, angle, translation, a smooth deformation grid, and pixels
with additive noise.
Their magnitudes scale with $E_b+d^s(\alpha)$, using respective coefficients
$3^\circ$, $0.035$, $0.055$, and $0.040$ for the last four perturbations.
This construction is deliberately heterogeneous: the two anchors contain errors that are neither
identical nor independent in every pixel, while their expected severity increases with source--target
pose distance.

\paragraph{Evaluation protocol.}
We evaluate the nine interior poses $\alpha\in\{0.1,\ldots,0.9\}$ at base distortion levels
$E_b\in\{0.2,0.4,0.5,0.6,0.8\}$. Each pose--distortion cell contains $256$ samples generated with a
$50$-step Heun solver. The same base Gaussian samples (seed $2027$) are reused across methods and
poses. PSNR and SSIM compare \emph{source-cond.} predictions with exact rotations of held-out MNIST
validation images. Pixel-space Fr\'echet distance (FD) compares \emph{joint-generated} target-slot
samples with the corresponding validation distribution using flattened, bilinearly downsampled
$16\times16$ pixel features; no Inception features are used. Each method--distortion setting currently
has one fine-tuning seed; variation across base distortion levels therefore measures robustness to
anchor quality rather than optimization uncertainty.

\FloatBarrier

\clearpage
\begingroup
\raggedbottom
\section{Additional NuScenes Qualitative Examples}
\label{app:nuscenes-qualitative}

Figures~\ref{fig:nuscenes-qualitative-0012-front-right}--\ref{fig:nuscenes-qualitative-0014-back-left}
provide additional qualitative comparisons of novel-view generation on nuScenes.
Rows show DiST-4D, OmniRe, GEN3C, FreeVS, and SymRegFlow from top to bottom;
columns show frames from each generated video.

\begin{figure}[H]
\centering
\includegraphics[width=0.8\textwidth]{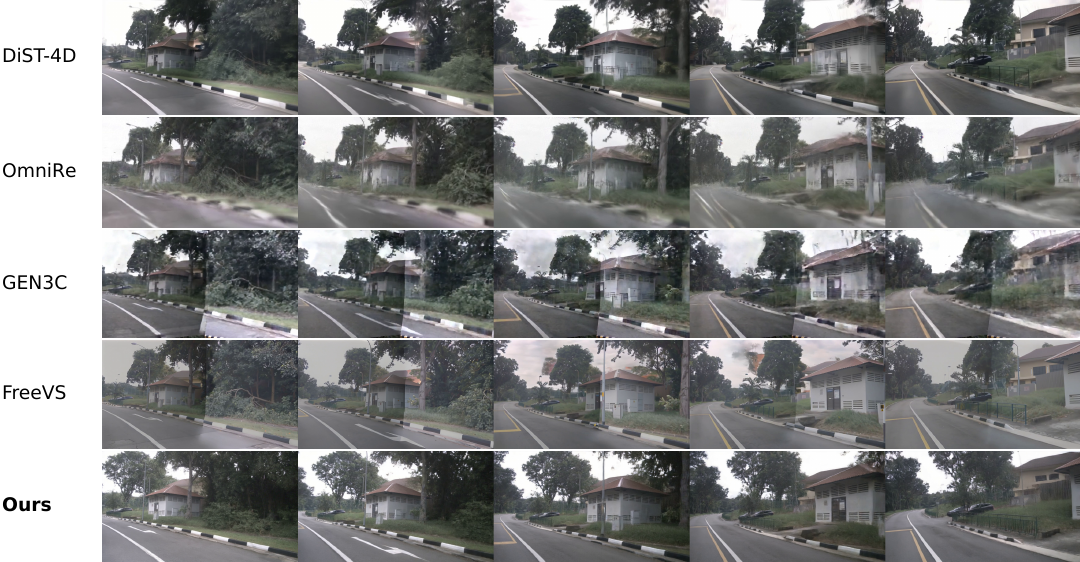}
\caption{Qualitative novel-view comparison on nuScenes scene-0012 between the front and
front-right cameras.}
\label{fig:nuscenes-qualitative-0012-front-right}
\end{figure}

\begin{figure}[H]
\centering
\includegraphics[width=0.8\textwidth]{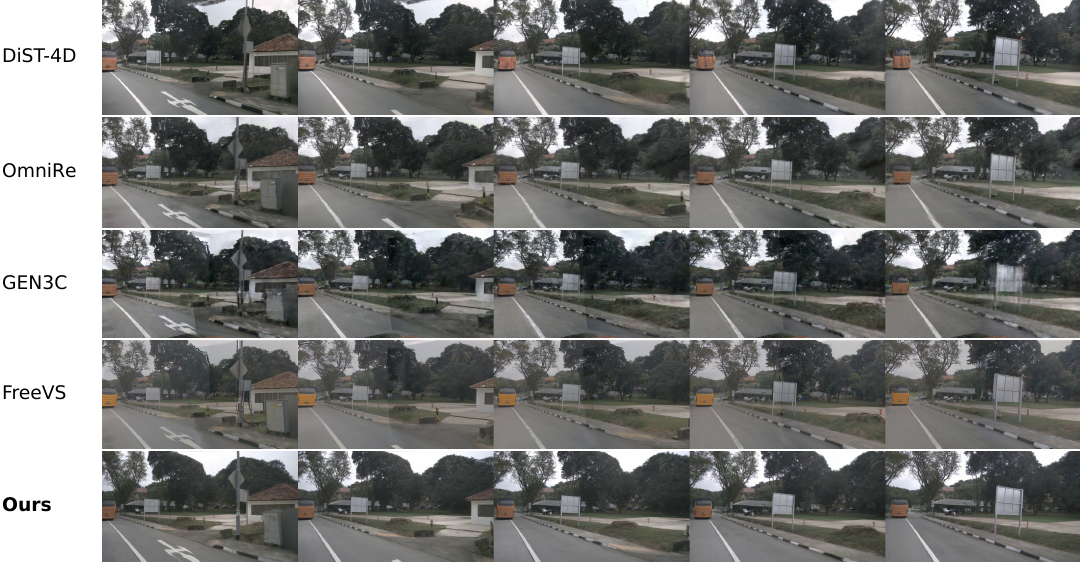}
\caption{Qualitative novel-view comparison on nuScenes scene-0013 between the front and
front-right cameras.}
\label{fig:nuscenes-qualitative-0013-front-right}
\end{figure}

\begin{figure}[H]
\centering
\includegraphics[width=0.8\textwidth]{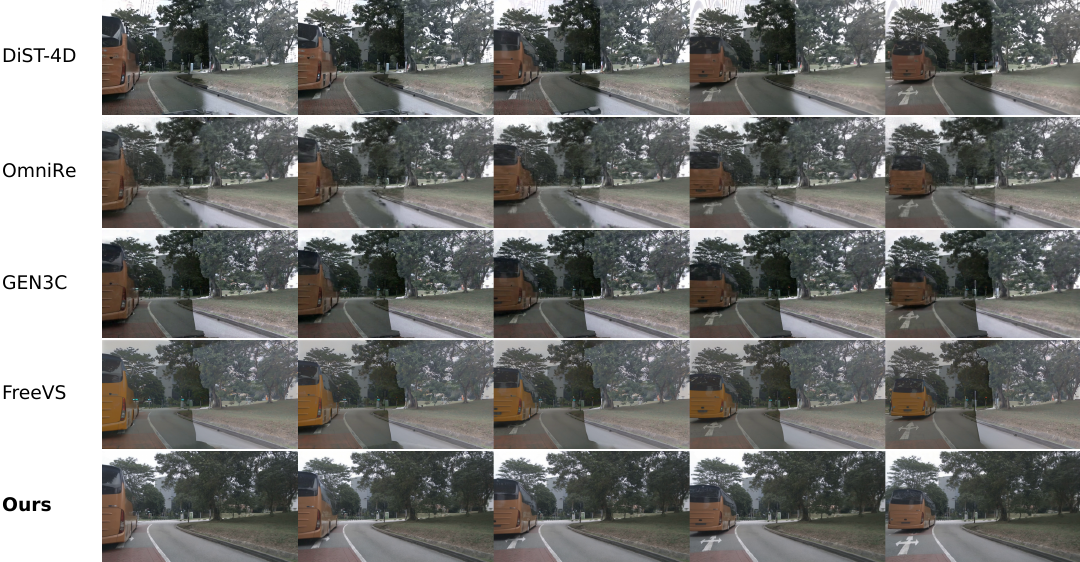}
\caption{Qualitative novel-view comparison on nuScenes scene-0014 between the front and
front-right cameras.}
\label{fig:nuscenes-qualitative-0014-front-right}
\end{figure}

\begin{figure}[H]
\centering
\includegraphics[width=0.8\textwidth]{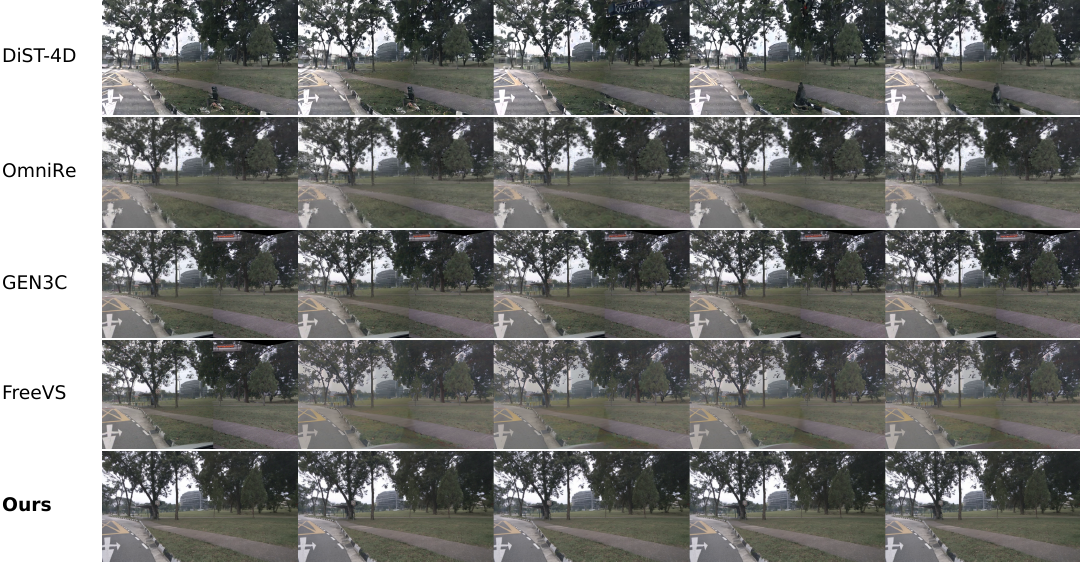}
\caption{Qualitative novel-view comparison on nuScenes scene-0014 between the back and
back-left cameras.}
\label{fig:nuscenes-qualitative-0014-back-left}
\end{figure}

\FloatBarrier
\clearpage
\endgroup

\end{document}